\documentclass{article} 
\usepackage[table]{xcolor}
\usepackage{iclr2025_conference}

\usepackage{times}
\usepackage{titletoc}
\usepackage[utf8]{inputenc} 
\usepackage[T1]{fontenc}    
\usepackage{hyperref}       
\usepackage{url}            
\usepackage{booktabs}       
\usepackage{amsfonts}       
\usepackage{nicefrac}       
\usepackage{microtype}      
\usepackage{mathrsfs}
\usepackage{graphicx}
\usepackage{subfigure}
\usepackage{multirow}
\usepackage{amsmath}
\usepackage{float}
\usepackage{amsthm}
\usepackage{comment}
\usepackage{multicol}
\usepackage{appendix}
\usepackage{amssymb}
\usepackage{wrapfig}
\usepackage{bbm}
\usepackage{bbding}
\usepackage{framed}
\usepackage{setspace}
\usepackage{pifont}
\usepackage{caption}

\usepackage{pythonhighlight}

\newcommand{\cmark}{\ding{51}}
\newcommand{\xmark}{\ding{55}}

\usepackage{makecell}
\usepackage[linesnumbered,ruled,lined]{algorithm2e}
\definecolor{mycolor1}{rgb}{0.82,0.70,0.54}
\definecolor{mycolor2}{rgb}{0.0,0.51,0.22}
\definecolor{mycolor3}{rgb}{0.80, 0.48, 0.37}
\definecolor{mycolor4}{rgb}{0.02, 0.33, 0.68}
\definecolor{mycolor5}{rgb}{0.86, 0.11, 0.11}

\newtheorem{lemma}{Lemma}
\newtheorem{theorem}{Theorem}

\usepackage{textcomp,booktabs}

\title{\texttt{COME}: Test-time adaption by Conservatively Minimizing Entropy}
\iclrfinalcopy

\author{Qingyang~Zhang\textsuperscript{1}, Yatao Bian\textsuperscript{2}, Xinke Kong\textsuperscript{1}, Peilin Zhao\textsuperscript{2} and Changqing Zhang\textsuperscript{1}
\\
College of Intelligence and Computing, Tianjin University\textsuperscript{1}
\\
Tencent AI Lab\textsuperscript{2}
}
\begin{document}

\maketitle

\begin{abstract}
Machine learning models must continuously self-adjust themselves for novel data distribution in the open world. As the predominant principle, entropy minimization (EM) has been proven to be a simple yet effective cornerstone in existing test-time adaption (TTA) methods. While unfortunately its fatal limitation (i.e., overconfidence) tends to result in model collapse. For this issue, we propose to \textbf{\texttt{Co}}nservatively \textbf{\texttt{M}}inimize the \textbf{\texttt{E}}ntropy (\texttt{COME}), which is a simple drop-in replacement of traditional EM to elegantly address the limitation. In essence, \texttt{COME} explicitly models the uncertainty by characterizing a Dirichlet prior distribution over model predictions during TTA. By doing so, \texttt{COME} naturally regularizes the model to favor conservative confidence on unreliable samples. Theoretically, we provide a preliminary analysis to reveal the ability of \texttt{COME} in enhancing the optimization stability by introducing a data-adaptive lower bound on the entropy. Empirically, our method achieves state-of-the-art performance on commonly used benchmarks, showing significant improvements in terms of classification accuracy and uncertainty estimation under various settings including standard, life-long and open-world TTA, i.e., up to $34.5\%$ improvement on accuracy and $15.1\%$ on false positive rate.
\end{abstract}
\section{Introduction}
Endowing machine learning models with self-adjust ability is essential for their deployment in the open world, such as autonomous vehicle control and embodied AI systems. To this end, test-time adaption (TTA) emerges as a promising strategy to enhance the performance in the open world which often encounters unexpected noise or corruption (e.g., data from rainy or snowy weather). Unsupervised losses play a crucial role in model adaptation, which can improve the accuracy of a model on novel distributional test data without the need for additional labeled training data. The representative strategy entropy minimization (EM) adapts classifiers by iteratively increasing the probabilities assigned to the most likely classes, and is an integral part in the state-of-the-art TTA methods~\citep{press2024entropy,wang2020tent, zhang2022memo, niu2022efficient, wang2022continual,iwasawa2021test,niu2023towards,yang2024test}. 
The initial intuition behind using entropy minimization, given by~\citep{wang2020tent} is based on the observation that models tend to be more accurate on samples for which they make predictions with higher confidence. The natural extension of this observation is to encourage models to bolster the confidence on test samples.

However, this intuition may not always be true since there always exists irreducible uncertainty which arises from the natural complexity of the data or abnormal outliers. Naturally, one might expect a machine learning model to adapt itself to test data and favor higher confidence on right prediction, but of course not absolute certainty for the erroneous. This contradiction challenges the suitability of EM in TTA tasks, which greedily pursues low-entropy on all test samples. A notable example in recent research concerns that EM can be highly unstable and frequently lead to model collapse when the models encounter unreliable samples in the wild~\citep{niu2023towards}. In this work, we hypothesize that due to the nature of EM, previous TTA methods tend to be highly overconfident ignoring the reliability of various test samples, which further results in the unsatisfactory performance.

For the above issues, we propose a simple yet effective model-agnostic learning principle, termed \textbf{Co}nservatively \textbf{M}inimizing \textbf{E}ntropy (\textbf{\texttt{COME}}) to stabilize TTA. We first consider the model output as \textit{opinion} which explicitly models the uncertainty of each sample from a Theory of Evidence perspective. Then, we encourage the model to favor definitive opinions for TTA and meanwhile take the uncertainty information into consideration. This offers two-fold advantages compared to EM learning principle. First, our \texttt{COME} leverages subjective logic~\citep{jsang2018subjective}, which is an off-the-shelf uncertainty tool in Bayesian toolbox to effectively perceive the uncertainty raised upon varying test samples without altering the original model architecture or training strategy. Second, when encountering unreliable outliers, the model is regularized to favor conservative confidence and be able to explicitly express \textit{"I do not know"}, i.e., reject to classify them to any known classes, which meets our expectation on model trustworthiness. Theoretically, our \texttt{COME} takes inspiration from Bayesian framework, and can be proved to correspond with a data-adaptive upper bound on the model confidence, which is a desirable property for TTA where the reliability of test samples are often varying from time. The contributions of this work are summarized as follows:
\begin{itemize}
\item As a principled alternative beyond entropy minimization, we propose a simple yet effective driven strategy for test-time adaption called Conservatively Minimizing Entropy (\texttt{COME}) which improves previous methods by exploring and exploiting the uncertainty.
\item We provide theoretical analysis with insight in contrast to EM, the model confidence of our \texttt{COME} is provably upper bounded in a data-adaptive manner, which enables TTA methods to focus on reliable samples and conservatively handle abnormal test samples.
\item We perform extensive experiments under various settings, including standard, open-world and lifelong TTA, where the proposed \texttt{COME} achieves excellent performance in terms of both classification accuracy and uncertainty quantification.
\end{itemize}

\section{Related work}
\textbf{Test-time adaption} aims to bridge the gaps between source and target domains during test-time without accessing the training-time source data. The model could be adapted by performing the unsupervised task on test samples. \textbf{Entropy minimization} performs an important role in test-time adaption, which has been integrated as a part of numerous TTA methods~\citep{press2024entropy,wang2020tent,niu2022efficient,wang2022continual,iwasawa2021test,yang2024test,chen2022contrastive}. However, it has been observed that the performance of EM can be highly sub-optimal and unstable when encounter unreliable environments. To this end, previous works incorporate many strategies including i) Samples selection, which selectively filter out the unreliable samples before adapting the model to them. For example,~\citep{iwasawa2021test,niu2023towards} manually set an entropy threshold and reject the high-entropy samples before model adaption. ii) Constrained optimization, which heuristically enforces that the updated parameters do not diverge too much compared to the original pretrain model during adaption. iii) Model recovery, which lively monitor the state of the adapting model and frequently reset it when detecting performance collapse~\citep{niu2023towards,wang2022continual}. Although these strategies have shown promising performance, the underlying reasons of the EM's sub-optimal performance are still largely unexplored. In contrast, this work aims to handle the inherent issues of EM overlooked by the existing studies and validate the necessity and effectiveness in various TTA settings.

\textbf{Uncertainty quantification} is one key aspect of the model reliability, which aims to quantitatively characterize the probability that predictions will be correct. With accurate uncertainty estimation ability, further processing can be taken to improve the performance of machine learning systems (e.g., human assistance) when the predictive uncertainty is high. This is especially useful in high-stake scenarios such as medical diagnosis~\citep{wang2023uncertainty}. To obtain the uncertainty, Bayesian neural networks (BNNs)~\citep{denker1990transforming,mackay1992bayesian} have been proposed to replace the deterministic weight parameters of model with distribution. Unlike BNNs, ensemble-based methods obtain the epistemic uncertainty by training multiple models and ensembling them~\citep{rahaman2021uncertainty, abe2022deep}. Uncertainty quantification has been successfully equipped to model the trustworthiness of varying environments in many fields such as multimodal learning~\citep{han2022trusted, zhang2023provable} and reinforcement learning~\citep{li2021mural,kalweit2017uncertainty}. In this paper, we focus on estimating and exploiting uncertainty under the theory of subjective logic (SL,~\citep{jsang2018subjective}). Unlike BNNs or ensemble, SL explicitly models the uncertainty in a single forward pass without modifying the training strategy or model architecture, which meets our expectation of computational effectiveness for TTA tasks.

\section{Motivation}
We consider the fully test-time adaption setting in $K$-classification task where $\mathcal{X}$ is the input space and $\mathcal{Y}=\{1,2,...,K\}$ denotes the target space. Given a classifier $f:\mathcal{X}\rightarrow \mathbb{R}^K$ parameterized by $\theta$ which has been pretrained on training distribution $P^{\rm train}$, our goal is to boost $f$ by updating its parameters $\theta$ online on each batch of test data drawn from test distribution $P^{\rm test}$. Note that in fully TTA setting, the training data $P^{\rm train}$ is inaccessible and one can only tune $\theta$ on unlabeled test data. This is derived from realistic concerns of privacy, bandwidth or profit.
\textbf{Entropy minimization (EM)} algorithm iteratively optimizes the model to minimize the predictive entropy on test sample $x$
\begin{equation}
    H(p(y|x))=-\sum_{k=1}^K p(y=k|x)\log p(y=k|x),
\end{equation}  
where $p(y|x)$ is the class distribution calculated by normalizing the output logits $f(x)$ with softmax function, i.e.,  $p(y=k|x)=\frac{\exp f_k(x)}{\sum \exp f(x)}$. $H$ is the Shannon's entropy.

\textbf{Other learning objectives.} Besides EM, there also exists several TTA methods which explore other unsupervised learning objectives. Notable examples include 1) Pseudo label (PL): $\mathcal{L}_{\rm PL}=-\mathbb{E} \log p(y=\hat{y}|x)$ which encourages the adapted model to fit the pseudo label $\hat{y}$ predicted by the pretrained model, 2) Module adjustment (T3A) which adjusts the parameters in the last fully connected layer, and can be viewed as an implicit way to minimize entropy~\citep{iwasawa2021test}, 3) Energy minimization: 
$\mathcal{L}_{\rm TEA}=-\mathbb{E} \log\sum_{k=1}^K \exp f(x)$ which aims to minimize the free energy during adaption, and takes inspiration from energy model~\citep{yuan2023energy}, 4) Contrastive learning objective: $\mathcal{L}_{\rm infoNCE}=-\log \frac{\exp {\rm query}\cdot {\rm key^{+}}}{\sum \exp {\rm query}\cdot {\rm key}}$ which strives to minimize the cosine distance between the query and positive samples ($\rm key^{+}$) while maximizing the cosine distances between query and negative samples~\citep{chen2022contrastive}, 5) The recent advanced FOA~\citep{niu2024test} which uses evolution strategy to minimize the test-training statistic discrepancy and model prediction entropy.

\textbf{The overconfident issue of EM.} We begin by testing EM in standard TTA setting, and put forward the following observations to detail its unsatisfying performance.

\begin{figure}[htbp]
    \centering
    \includegraphics[width=0.99\textwidth]{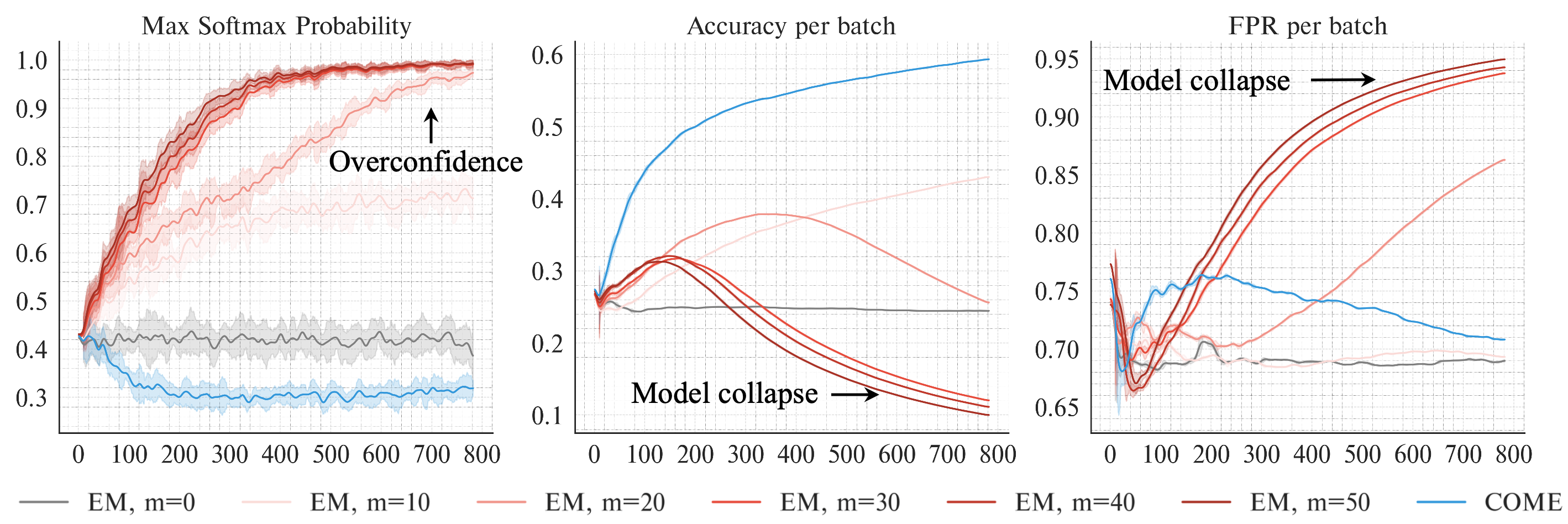}
    \caption{Empirical observations of Entropy Minimization when equipped to Tent~\citep{wang2020tent}. Along the TTA process, the uncertainty of models tuned with EM quickly drops, and the false positive rate decreases temporarily for a very short time horizon before quickly increasing. Along the same adaption trajectory, the model accuracy also improves for a short time compared to the initial model and then quickly decreases, after which the model collapses to a trivial solution. We manually tune an entropy threshold to filter out a proportion of (100-m)\% unreliable samples with highest entropy and only conduct entropy minimization on the rest m\% low-entropy samples. However, the resultant methods still suffer from aforementioned issues. Therefore, we believe that the entropy minimization learning principle is inherently problematic in TTA, which necessitates a more principled solution.}
    \label{fig:motivation1}
\end{figure}

\begin{figure}[htbp]
    \centering
    \includegraphics[width=0.99\textwidth]{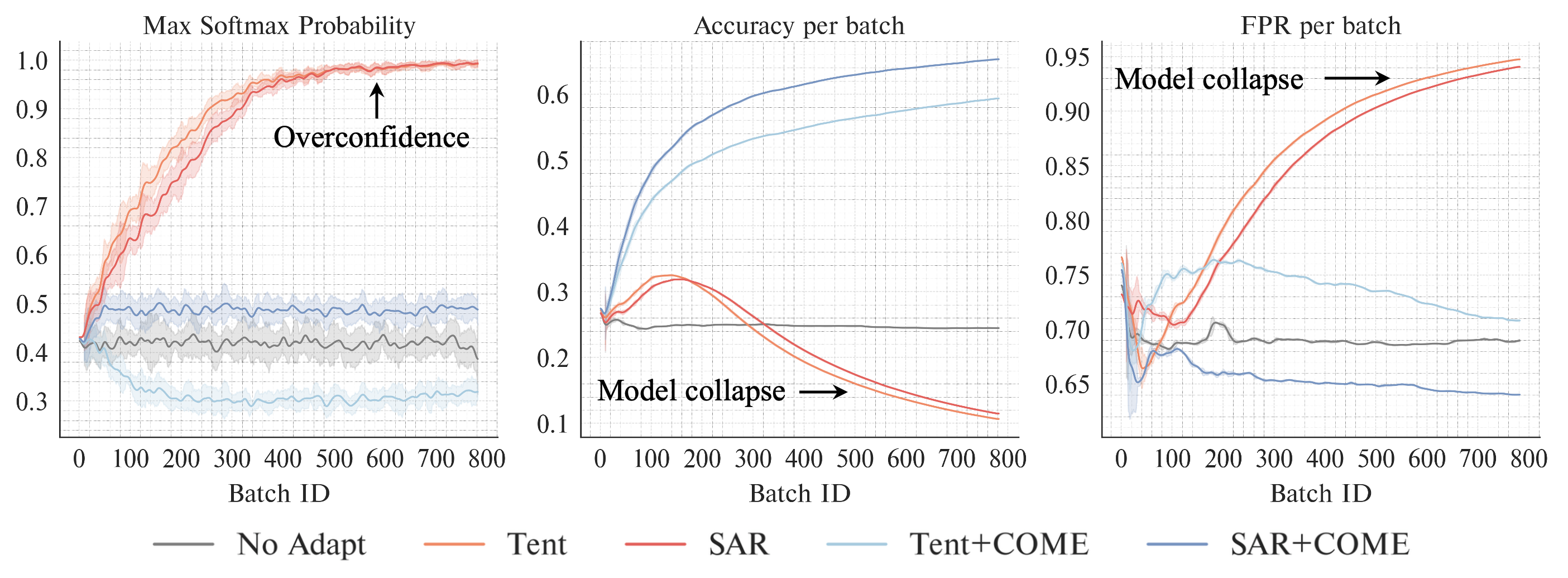}
    \caption{Comparison on two representative TTA methods, i.e., the seminal Tent~\citep{wang2020tent} and recent SOTA SAR~\citep{niu2023towards}. By contrast to EM, our \texttt{COME} establishes a stable TTA process with consistently improved classification accuracy and false positive rate. Besides, the model confidence of our \texttt{COME} is much more conservative, which implies fewer risks of overconfidence and a more accurate uncertainty awareness.}
    \label{fig:motivation2}
    \vspace{-10pt}
\end{figure}

As shown in Figure~\ref{fig:motivation1}, after TTA, EM tends to give overconfident prediction and assign extremely high probability to one certain class. While we believe that tuning the model to favor confident prediction can boost classification accuracy, the unlimited low entropy on all test samples is obviously a rather undesired characteristic. The most straightforward way to overcome this limitation is to filter out unreliable samples. However, similar issues remain in the resultant methods. This urges the demand of a more principled solution demonstrated in Figure~\ref{fig:motivation2}, where the proposed \texttt{COME} explicitly models uncertainty and regularizes the model for conservative predictive confidence during TTA. We test on ImageNet-C under snow corruption of severity level 5 as a typical showcase, and refer interested readers to Appendix~\ref{visualization} for more similar results.

\section{Methodology}
We propose to conservatively minimize the entropy under uncertainty modeling, a simple alternative to EM algorithm. The key idea of \texttt{COME} is to quantify and then regularize the uncertainty during TTA without altering the model architecture or training strategy, which avoids the overconfident nature of EM at minimal cost. We first introduce uncertainty quantification by the theory of evidence which is the the fundamental block of \texttt{COME} and then present how to regularize the uncertainty during TTA.

\subsection{Modeling Uncertainty by the Subjective Logic}
To overcome the greediness of EM, we need to effectively perceive the trustworthiness of diverse test samples firstly. Given a well trained classifier $f:\mathcal{X}\rightarrow \mathcal{Y}$, the most simple way to quantify the uncertainty of each sample is using the softmax probability as confidence in prediction. A few pioneer works propose to filter out the test samples with high-entropy predicted softmax probability for stable TTA~\citep{niu2023towards,niu2022efficient}. However, it has been shown that softmax probability often leads to overconfident predictions, even when the predictions are wrong or the inputs are abnormal outliers~\citep{moon2020confidence,van2020uncertainty}. Thus this simple strategy may not be satisfied enough and highlights the necessity of better uncertainty modeling. To this end, we propose to obtain the uncertainty through the theory of subjective logic, which defines a framework for obtaining the probabilities (belief masses) of different classes and the overall uncertainty (uncertainty mass) based on the \textit{evidence} \footnote{In Bayesian context, evidence refers to the metrics collected from the input to support the classification.} collected from data. Specifically, in $K$ classification task, SL formalizes the belief assignments over a frame of discernment as a Dirichlet distribution. In contrast to softmax function that directly normalizes the model output logits $f(x)$ to model predictive class distribution $p(y|x)$, SL considers the model output as evidence (denoted as $\boldsymbol{e}$) to model a Dirichlet distribution which represents the density of all possible probability assignment $\boldsymbol{\mu}=[p(y=1|x),p(y=2|x),\dots,p(y=K|x)]$. That is, the predicted categoricals $\boldsymbol{\mu}$ is also a random variable itself, which yields a Dirichlet distribution as follow
\begin{equation}
\label{dirichlet}
    p(\boldsymbol{\mu}|x)={\rm Dir}(\boldsymbol{\mu}|\boldsymbol{\alpha})=\frac{1}{B(\boldsymbol{\alpha})}\prod_{k=1}^K \mu^{\alpha_k-1}_k, \ \boldsymbol{\alpha}= \boldsymbol{e}+1,
\end{equation}
where ${\rm Dir}(\boldsymbol{\mu}|\boldsymbol{\alpha})$ is the Dirichlet distribution characterized by parameters $\boldsymbol{\alpha}$. The summation of all $\alpha_k\in\boldsymbol{\alpha}$ is so called the strength $S$ of the Dirichlet distribution, i.e., $S=\sum_{k}\alpha_k=\sum_{k} e_k+1$. Then SL tries to assign a belief mass $b_k$ to each class label $k$ and an overall uncertainty mass $u$ to the whole frame based on the collected evidence as follow
\begin{equation}
\label{SL}
    b_k=\frac{e_k}{S}=\frac{\alpha_k-1}{S}\ {\rm and} \ u=\frac{K}{S}, \ {\rm subject\ to} \ u+\sum_{k=1}^K b_k =1,
\end{equation}
where $S$ is the Dirichlet strength which denotes the total evidence we collected and $K$ is the total classes number. Eq.~\ref{SL} actually describes the phenomenon where the more evidence observed for the $k$-th category, the greater the belief mass assigned to the $k$-th class. Correspondingly, the less total evidence $S$ observed, the greater the total uncertainty $u$. Such assignment is so called the subjective opinion
\begin{equation}
    \mathcal{M}(x)=[b_1,b_2,\cdots,b_k,u],
\end{equation}
which not only describes the belief of assigning $x$ to each class $k$ but also explicitly models the uncertainty due to lack-of-evidence. At a high-level, the Dirichlet distribution can be considered as the \textit{distribution of distribution}~\citep{malinin2018predictive} parametrized over evidence, which represents the density of all the possible class probability assignments. Hence it models second-order probabilities and uncertainty.

Given an opinion, the final predicted probability is calculated by taking the expectation of the corresponding Dirichlet distribution and computed as
\begin{equation}
    p(y=k|x)=\int p(y=k|\boldsymbol{\mu})p(\boldsymbol{\mu}|x)d\boldsymbol{\mu}=\frac{\alpha_k}{S}.
\end{equation}
If an exponential output function is used for obtaining the parameters of Dirichlet distribution from the model output $f(x)$ (i.e., logits), where $\boldsymbol{\alpha}=\exp f(x)$ and $\boldsymbol{e}=\boldsymbol{\alpha}-1$, then the expected posterior probability of a label $k$, i.e., $p(y=k|x)$ is calculated in the same way with standard softmax output function~\citep{malinin2018predictive}. Given the above analysis, standard DNNs for classification with a softmax output function can be viewed as predicting the expected categorical distribution under a Dirichlet prior. And thus we can directly mode uncertainty with subjective logic for a model pretrained with softmax function and standard cross-entropy loss, without modifying its architecture or training strategy. We refer interested readers to~\cite{malinin2018predictive} for more details of implementation.

\textbf{Benefits of modeling uncertainty based on subjective opinion.} It has been widely recognized that using the softmax output as the confidence often leads to overconfidence phenomenon~\citep{guo2017calibration,hendrycks2016baseline,liu2020energy}. Other advanced uncertainty measurement methods such as MC-dropout~\citep{gal2016dropout}, deep ensemble~\citep{lakshminarayanan2017simple,rahaman2021uncertainty}, calibration~\citep{guo2017calibration,han2024selective} usually require additional computations during inference or a separated validation set, which can not be seamlessly integrated into existing TTA methods. By contrast, the introduced SL is model-agnostic, which can avoid these issues by constructing a Dirichlet prior over the model output and directly deducing an additional uncertainty mass through one single forward pass. Besides these conveniences, we empirically show that the SL obtains more reliable uncertainty quantification than softmax probability in the Appendix, which is consistent with existing works~\cite{sensoy2018evidential,malinin2018predictive}.

\subsection{Model adaption by sharpening the Opinion}
Vanilla EM minimizes the softmax entropy of the predicted class distribution $p(y|x)$, which inevitably results in assigning rather high probability to one certain class. In contrast, we propose the learning principle to minimize the entropy of opinion
\begin{equation}
    \underset{\theta}{\rm minimize}\  H(\mathcal{M}(x))=-\sum_{k=1}^K b_k\log b_k-u\log u.
\end{equation}
Compared to entropy minimization on softmax probability which ultimately assigns all the probability to one certain class, the above learning principle offers the model with an additional option, i.e., express high overall uncertainty and reject to classify when the observed total evidence is insufficient. Since subjective logic provides an additional uncertainty mass function $u(x)$ in the opinion. In other words, by assigning all belief masses (probability) to uncertainty $u$, the model can now express ``\textit{I do not know}'' as its predicted opinion.

\subsection{Regularizing uncertainty in an unsupervised manner}
While subjective logic offers an opportunity to modeling uncertainty and reject to classify unreliable samples, naively minimizing the entropy of opinion for TTA may still be problematic. As shown in previous works, model pretrained with softmax output function frequently suffers from overconfidence issue~\citep{guo2017calibration,nguyen2015deep,hendrycks2018deep}. Therefore, the belief mass assigned to the one certain class $k$ by the pretrained model is usually much larger than the uncertainty mass $u$. This results in the model tendency of increasing the belief mass during the entropy minimization process, while neglecting the uncertainty function \textit{u}. Motivated by the above analysis, our next goal is to devise an effective regularization strategy for the uncertainty mass. In supervised learning tasks, previous works leverage labeled training data to constrain the uncertainty mass~\citep{sensoy2018evidential,malinin2018predictive}. However, these strategies is inapplicable due to the unsupervised nature of TTA task where the training data is unavailable. This motivates us to explore the uncertainty information lies in the pretrained model itself for regularization without additional supervision. As one of the simplest yet effective design choices, we propose to constrain the uncertainty mass predicted by the adapted model not to diverge too far from the pretrained model. This results in the following constrained optimization objective
\begin{equation}
\label{constrained-op}
    \underset{\theta}{\rm minimize}\  H(\mathcal{M}(x))\ \ {\rm subject\ to}\ |u_{\theta}(x)-u_{\theta_0}(x)|\leq \delta,
\end{equation}
where $\theta$, $\theta_0$ denote the adapted and pretrained model respectively, and $u$ is the uncertainty estimated by Eq.~\ref{SL} and $\delta$ is a threshold. Considering the difficulty of constrained optimization in modern neural networks, our next target is to find a way to convert Eq.~\ref{constrained-op} into an unconstrained form. To this end, we introduce the following Lemma.
\begin{lemma} For any $x\in\mathcal{X}$, we have
    \begin{equation}
      \frac{K}{||f(x)||_p+\log K}\leq u\leq \frac{K^{1+1/p}}{||f(x)||_p},
    \end{equation}
where $f(x)$ is the model output logits, $K$ is the total class number and $||\cdot||_p$ denotes the $p$-norm. 
\end{lemma}
Lemma 1 shows that the uncertainty mass of subjective opinion is bounded by the norm of the total evidence collected from the model output. Thus instead of directly constraining on $u(x)$, we can alternatively constrain on the $p$-norm of model output logits, which is more flexible. Taking inspiration from previous work in supervised learning literature~\citep{wei2022mitigating}, this can be achieved by factorize $f(x)$ into $f(x)/||f(x)||_p \cdot ||f(x)||_p$ and then enforcing the gradient on the second term to be equal to zero during optimization. Specifically, the final minimizing objective of \texttt{COME} is
\begin{equation}
   \underset{\theta}{\rm minimize}\  H(\mathcal{M}(x))\ \ {\rm where}\ f(x)= \frac{f(x)}{||f(x)||_p}\cdot ||f(x)||_p^{\rm no\_grad}\cdot \tau,
\end{equation}
and $||f(x)||_p^{\rm no\_grad}$ is the $p$-norm of $f(x)$ with zero gradient. This can be achieved by applying the detach operation which is a common used function in modern deep learning toolbox like PyTorch and TensorFlow. By doing so, minimizing the entropy of opinion would not influence $||f(x)||_p$. $\tau$ is a hyper-parameter which controls the magnitude of recovered logits. In our experiments, we choose $p=2$ and $\tau=1$ for simplicity. Our \texttt{COME} can be implemented by modifying only a few lines of code in the original EM algorithm (shown as Algorithm~\ref{alg1}).

\begin{algorithm}[H]
\caption{Pseudo code of \texttt{COME} in a PyTorch-like style.}
\label{alg1}
\begin{lstlisting}
# x: the output logits, model: the test model
def entropy_of_opinion(x):
    belief = exp(x) - 1 / sum(exp(x)) # belief mass
    uncertainty = K / sum(exp(x)) # uncertainty mass
    opinion = cat([belief, uncertainty]) # subjective opinion
    return -sum(opinion * log(opinion)) # entropy of opinion

for data in test_loader: # load a minibatch data
    x = model(data) # forward
    x = x / norm(x, p=2) * norm(x, p=2).detach() # constraint in Eq.9
    loss = entropy_of_opinion(x) # calculate loss
    # ... [backwards and update the parameters]
\end{lstlisting}
\end{algorithm}

\textbf{Stability of \texttt{COME}.} We provide preliminary theoretical understanding of the superiority of \texttt{COME}. As we mentioned before, one notable limitation of EM is that it enforces low entropy for all test samples while ignores the instinct complexity of wild test data. Thus at the end of TTA progress, EM ultimately produces model that yields overconfident prediction. Our \texttt{COME} resolves this issue and introduces an upper bound for each test sample $x$ according to its trustworthiness. This property is formalized as follows
\begin{theorem}[Model confidence upper bound]
\label{theorem:ent-lb}
For any $x\in\mathcal{X}$, if $|u(x)- u_0(x)|\leq \delta$ holds, then we have
    \begin{equation}
        \underset{k}{\rm max}\ p(y=k|x)\leq \frac{1}{1+(K-1)\exp{(-\frac{K}{u_0-\delta})}},
    \end{equation}
where $\max\limits_k p(y=k|x)$ is the model confidence (class probability assigned to the most likely class) and $K$ is the total class number. $u_0$ is the shorthand of $u_{\theta_0}(x)$.
\end{theorem}
From Theorem~\ref{theorem:ent-lb}, we find that the model confidence in \texttt{COME} has a sample-wise upper bound according to $u_0(x)$. In particular, it implies that the model confidence upper bound of the most likely class decreases according to $u_0(x)$. For this reason, one can suspect that if the test model is uncertain about some sample $x$ (with a rather large $u_0$), it will be difficult to further increase the model confidence on such $x$, which is a desirable property for TTA in the wild.

\section{Experiments}
We conduct experiments on multiple datasets with distributional shift to answer the following questions. Q1. In the standard TTA setting, does the proposed method outperform other algorithms? Q2. How does \texttt{COME} perform in various settings, such as open-world TTA or lifelong TTA? Q3. Uncertainty quantification is both the motivation behind \texttt{COME} and the reason for its effectiveness, does our method achieves more reliable uncertainty estimation during TTA? Q4. Ablation study - what is the key factor of performance improvement in our method?

\subsection{setup}
\textbf{Datasets}. Following the common practice~\citep{niu2022efficient,niu2023towards}, we perform experiments under both standard covariate-shifted distribution dataset ImageNet-C, a large-scale benchmark with 15 types of diverse corruption belong to 4 main categories (noise, blur, weather and digital). Besides, we also consider open-world test-time adaption setting, where the test data distribution $P^{\rm test}$ is a mixture of both normal covariate-shifted data $P^{\rm Cov}$ and abnormal outliers $P^{\rm Outlier}$ of which the true labels do not belong to any known classes in $P^{\rm train}$. Following previous work in outlier detection literature, $P^{\rm Outlier}$ is a suit of diverse datasets introduced by~\citep{yang2022openood}, including iNaturalist, Open-Image, NINCO and SSB-Hard.
\textbf{Compared methods.} We compare our \texttt{COME} with a board line of test-time adaption methods, including both EM-based and non-EM methods. $\circ$ EM-based methods choose entropy minimization as their learning objective, including Tent~\citep{wang2020tent}, EATA~\citep{niu2022efficient}, CoTTA~\citep{wang2022continual} and recent advanced SAR~\citep{niu2023towards}. $\circ$ Non-EM methods employ other learning objectives including Pseudo Label (PL), module adjustment~\citep{iwasawa2021test} (T3A) and energy minimization~\citep{yuan2023energy} (TEA). To be consistent with previous works~\citep{niu2023towards}, we use the ViT-base architecture as our backbone and refer insterested readers to the Appendix for results on ResNet. The test batch size is 64. When equipped to previous EM-based TTA baselines, we only replace the learning objective with our \texttt{COME} and keep all the other configures unchanged (consistent to the official implementation).
\textbf{Tasks and Metrics.} For classification performance comparison, we report the accuracy (Acc) on covariate-shifted data. Besides, for uncertainty estimation evaluation, we report the average false positive rate (FPR). The mis-classified samples and outliers are considered as positive samples which should be of higher uncertainty compared to correct classification that is considered as negative.

\subsection{Experimental results}

\textbf{Performance comparison in standard TTA settings (Q1).} As shown in Table \ref{tab:standard}, our \texttt{COME} establishes strong overall performance in terms of both classification and uncertainty estimation tasks. We highlight a few essential observations. Compared to EM learning principle, our \texttt{COME} consistently outperforms it when equipped to the same baseline methods, including Tent~\citep{wang2020tent}, EATA~\citep{niu2022efficient}, CoTTA~\citep{wang2022continual} and SAR~\citep{niu2023towards}. As an example of our method's improved performance, when equipped to the recent SAR, our method yields an accuracy of $64.2\%$ and FPR95 of $63.8\%$, which outperforms the original implementation based on EM of \textbf{$10.1\%$} and \textbf{$2.9\%$} in terms of accuracy and FPR95 respectively. Besides, we also compare to Non-EM TTA methods, including TEA, T3A and PL. These methods do not rely on EM learning objective, yet are less effective than EM in terms of classification and uncertainty estimation performance.
\begin{table}[!htbp]
\centering
\vspace{-5pt}
\caption{Classification accuracy comparison on ImageNet-C (level 5). Substantial ($\geq 0.5$) \textcolor{mycolor4}{improvement} and \textcolor{mycolor3}{degradation} compared to the baseline are highlighted in blue or brown respectively. We only report average FPR$\downarrow$ here and defer the detailed results on each corruption to Appendix \ref{fprs}.}
\label{tab:standard}
\resizebox{\textwidth}{!}{
\setlength{\tabcolsep}{0.75mm}
\begin{tabular}{cc|ccc|cccc|cccc|cccc|cc}
\toprule
 & &  \multicolumn{3}{c}{Noise} & \multicolumn{4}{c}{Blur} & \multicolumn{4}{c}{Weather} & \multicolumn{4}{c}{Digital} & \multicolumn{2}{c}{Avg.}  \\ 
\cmidrule(r){3-5} \cmidrule(r){6-9} \cmidrule(r){10-13} \cmidrule(r){14-17}  \cmidrule(r){18-19}  
 Methods & \texttt{COME} & Gauss. & Shot & Impul. & Defoc & Glass & Motion & Zoom & Snow & Frost &Fog & Brit. &Contr. &Elast. &Pixel & JPEG & Acc$\uparrow$&FPR$\downarrow$ \\
 \toprule
No Adapt
&\xmark   &35.1 & 32.2 & 35.9 & 31.4 & 25.3 & 39.4 & 31.6 & 24.5 & 30.1 & 54.7 & 64.5 & 49.0 & 34.2 & 53.2 & 56.5 &39.8&67.5\\
\midrule
PL
&\xmark   &49.9 & 48.6 & 51.1 & 48.0 & 41.0 & 52.8 & 43.1 & 22.6 & 40.7 & 63.8 & 73.0 & 65.4 & 43.5 & 63.9 & 62.7 & 51.3&69.1\\

T3A
&\xmark   &34.5 & 31.5 & 35.4 & 32.6 & 27.5 & 40.8 & 33.6 & 25.7 & 31.0 & 56.4 & 64.9 & 50.8 & 37.9 & 54.3 & 58.4 & 41.0&67.7\\

TEA
&\xmark   &44.6 & 39.2 & 45.9 & 38.0 & 36.0 & 46.5 & 38.4 & 9.1 & 46.7 & 59.9 & 72.3 & 59.9 & 45.6 & 62.4 & 59.0 & 46.9&68.3\\

LAME
&\xmark   &34.8 & 31.9 & 35.5 & 31.0 & 24.4 & 39.0 & 30.7 & 23.4 & 29.6 & 53.3 & 64.2 & 40.9 & 32.7 & 52.8 & 56.0 & 38.7&69.7\\

FOA
&\xmark   &47.4 & 42.6 & 48.2 & 47.6 & 40.3 & 49.4 & 42.8 & 54.5 & 52.4 & 65.9 & 76.4 & 63.2 & 49.0 & 61.5 & 64.0 & 53.7&63.6\\

\midrule \multirow{3}{*}{Tent}
&\xmark   &52.6 & 52.1 & 53.5 & 52.9 & 47.7 & 56.7 & 47.5 & 10.5 & 28.6 & 67.2 & 74.4 & 67.3 & 50.7 & 66.3 & 64.6 & 52.8&70.1\\
&\cmark  &53.9 & 53.9 & 55.3 & 55.9 & 51.9 & 59.8 & 52.6 & 58.7 & 61.2 & 71.3 & 78.2 & 68.9 & 58.0 & 70.5 & 68.2 & 61.2&66.5\\
&Improve  &\textcolor{mycolor4}{$\bigtriangleup1.2$} & \textcolor{mycolor4}{$\bigtriangleup1.8$} & \textcolor{mycolor4}{$\bigtriangleup1.8$} & \textcolor{mycolor4}{$\bigtriangleup3.1$} & \textcolor{mycolor4}{$\bigtriangleup4.3$} & \textcolor{mycolor4}{$\bigtriangleup3.1$} & \textcolor{mycolor4}{$\bigtriangleup5.1$} & \textcolor{mycolor4}{$\bigtriangleup48.2$} & \textcolor{mycolor4}{$\bigtriangleup32.5$} & \textcolor{mycolor4}{$\bigtriangleup4.1$} & \textcolor{mycolor4}{$\bigtriangleup3.8$} & \textcolor{mycolor4}{$\bigtriangleup1.6$} & \textcolor{mycolor4}{$\bigtriangleup7.3$} & \textcolor{mycolor4}{$\bigtriangleup4.2$} & \textcolor{mycolor4}{$\bigtriangleup3.6$} &\cellcolor{gray!20} \textcolor{mycolor4}{$\bigtriangleup8.4$}&\cellcolor{gray!20}\textcolor{mycolor4}{$\bigtriangledown3.6$}\\
\midrule \multirow{3}{*}{EATA}
&\xmark   &55.7 & 56.5 & 57.2 & 57.3 & 53.2 & 58.3 & 58.6 & 61.8 & 60.1 & 71.1 & 75.3 & 68.5 & 62.5 & 68.7 & 66.3 & 62.1&65.1\\
&\cmark  &56.2 & 56.6 & 57.2 & 58.1 & 57.6 & 62.5 & 59.5 & 65.5 & 63.9 & 72.5 & 78.1 & 69.7 & 66.5 & 72.4 & 70.7 & 64.5&63.8\\
&Improve  &\textcolor{mycolor4}{$\bigtriangleup0.5$} & \textcolor{mycolor4}{$\bigtriangleup0.1$} & \textcolor{mycolor4}{$\bigtriangleup0.0$} & \textcolor{mycolor4}{$\bigtriangleup0.9$} & \textcolor{mycolor4}{$\bigtriangleup4.3$} & \textcolor{mycolor4}{$\bigtriangleup4.2$} & \textcolor{mycolor4}{$\bigtriangleup0.9$} & \textcolor{mycolor4}{$\bigtriangleup3.7$} & \textcolor{mycolor4}{$\bigtriangleup3.8$} & \textcolor{mycolor4}{$\bigtriangleup1.4$} & \textcolor{mycolor4}{$\bigtriangleup2.8$} & \textcolor{mycolor4}{$\bigtriangleup1.2$} & \textcolor{mycolor4}{$\bigtriangleup4.0$} & \textcolor{mycolor4}{$\bigtriangleup3.8$} & \textcolor{mycolor4}{$\bigtriangleup4.4$} & \cellcolor{gray!20}\textcolor{mycolor4}{$\bigtriangleup2.4$}&\cellcolor{gray!20}\textcolor{mycolor4}{$\bigtriangledown1.3$}\\
\midrule \multirow{3}{*}{SAR}
&\xmark   &51.8 & 51.7 & 52.9 & 50.8 & 48.6 & 55.3 & 49.2 & 23.3 & 46.5 & 65.6 & 73.0 & 65.8 & 51.1 & 64.0 & 62.6 & 54.2&66.7\\
&\cmark  &56.3 & 56.5 & 57.4 & 58.6 & 57.0 & 62.8 & 58.4 & 65.2 & 64.3 & 72.8 & 78.5 & 69.6 & 64.3 & 71.9 & 69.6 & 64.2&63.8\\
&Improve  &\textcolor{mycolor4}{$\bigtriangleup4.5$} & \textcolor{mycolor4}{$\bigtriangleup4.8$} & \textcolor{mycolor4}{$\bigtriangleup4.5$} & \textcolor{mycolor4}{$\bigtriangleup7.8$} & \textcolor{mycolor4}{$\bigtriangleup8.3$} & \textcolor{mycolor4}{$\bigtriangleup7.5$} & \textcolor{mycolor4}{$\bigtriangleup9.2$} & \textcolor{mycolor4}{$\bigtriangleup42.0$} & \textcolor{mycolor4}{$\bigtriangleup17.8$} & \textcolor{mycolor4}{$\bigtriangleup7.2$} & \textcolor{mycolor4}{$\bigtriangleup5.4$} & \textcolor{mycolor4}{$\bigtriangleup3.9$} & \textcolor{mycolor4}{$\bigtriangleup13.2$} & \textcolor{mycolor4}{$\bigtriangleup7.9$} & \textcolor{mycolor4}{$\bigtriangleup7.0$} & \cellcolor{gray!20}\textcolor{mycolor4}{$\bigtriangleup10.1$}&\cellcolor{gray!20}\textcolor{mycolor4}{$\bigtriangledown2.9$}\\
\midrule \multirow{3}{*}{CoTTA}
&\xmark   &40.6 & 37.8 & 41.7 & 33.7 & 29.5 & 43.8 & 35.6 & 38.1 & 43.3 & 59.2 & 70.5 & 59.3 & 40.1 & 57.9 & 59.7 & 46.1&67.9\\
&\cmark  &43.5 & 41.4 & 45.4 & 36.8 & 29.6 & 47.6 & 38.2 & 42.1 & 42.7 & 62.4 & 73.4 & 62.9 & 43.0 & 63.2 & 63.7 & 49.1&67.5\\
&Improve  &\textcolor{mycolor4}{$\bigtriangleup2.9$} & \textcolor{mycolor4}{$\bigtriangleup3.6$} & \textcolor{mycolor4}{$\bigtriangleup3.7$} & \textcolor{mycolor4}{$\bigtriangleup3.1$} & \textcolor{mycolor4}{$\bigtriangleup0.1$} & \textcolor{mycolor4}{$\bigtriangleup3.8$} & \textcolor{mycolor4}{$\bigtriangleup2.5$} & \textcolor{mycolor4}{$\bigtriangleup4.0$} & \textcolor{mycolor1}{$\bigtriangledown0.6$} & \textcolor{mycolor4}{$\bigtriangleup3.2$} & \textcolor{mycolor4}{$\bigtriangleup2.8$} & \textcolor{mycolor4}{$\bigtriangleup3.6$} & \textcolor{mycolor4}{$\bigtriangleup2.9$} & \textcolor{mycolor4}{$\bigtriangleup5.3$} & \textcolor{mycolor4}{$\bigtriangleup4.0$} &\cellcolor{gray!20} \textcolor{mycolor4}{$\bigtriangleup3.0$}&\cellcolor{gray!20}\textcolor{mycolor4}{$\bigtriangledown0.3$}\\
\midrule \multirow{3}{*}{MEMO}
&\xmark   &39.7 & 36.5 & 39.8 & 32.4 & 25.8 & 40.3 & 34.7 & 27.5 & 32.8 & 53.5 & 66.2 & 56.0 & 35.7 & 55.9 & 58.2 & 42.3&72.1\\
&\cmark  &40.6 & 37.5 & 40.6 & 33.4 & 26.7 & 41.2 & 35.4 & 28.7 & 33.7 & 54.7 & 67.1 & 55.9 & 36.6 & 57.2 & 59.3 & 43.2&70.8\\
&Improve  &\textcolor{mycolor4}{$\bigtriangleup0.8$} & \textcolor{mycolor4}{$\bigtriangleup1.0$} & \textcolor{mycolor4}{$\bigtriangleup0.8$} & \textcolor{mycolor4}{$\bigtriangleup1.0$} & \textcolor{mycolor4}{$\bigtriangleup0.9$} & \textcolor{mycolor4}{$\bigtriangleup1.0$} & \textcolor{mycolor4}{$\bigtriangleup0.7$} & \textcolor{mycolor4}{$\bigtriangleup1.2$} & \textcolor{mycolor4}{$\bigtriangleup0.9$} & \textcolor{mycolor4}{$\bigtriangleup1.2$} & \textcolor{mycolor4}{$\bigtriangleup0.8$} & \textcolor{mycolor1}{$\bigtriangledown0.1$} & \textcolor{mycolor4}{$\bigtriangleup0.9$} & \textcolor{mycolor4}{$\bigtriangleup1.3$} & \textcolor{mycolor4}{$\bigtriangleup1.1$} &\cellcolor{gray!20} \textcolor{mycolor4}{$\bigtriangleup0.9$}&\cellcolor{gray!20}\textcolor{mycolor4}{$\bigtriangledown1.3$}\\
\bottomrule
\end{tabular}
}
\vspace{-10pt}
\end{table}

\begin{table}[!htbp]
    \centering
    \caption{Classification and uncertainty estimation comparisons under \textbf{open-world} TTA settings, where $P^{\rm test}$ is a mixture of both covariate-shifted samples (Gaussian noise of severity level 3) and a suit of diverse abnormal outliers. Additional results with various mix ratios are in Appendix \ref{open-world-0.3}.}
    \label{tab:openworld}
    \resizebox{\textwidth}{!}{
    \setlength{\tabcolsep}{1.75mm}
    \begin{tabular}{cc|cc|cc|cc|cc|cc|cc|cc}
    \toprule
      & & \multicolumn{2}{c}{None} & \multicolumn{2}{c}{NINCO} & \multicolumn{2}{c}{iNaturist} & \multicolumn{2}{c}{SSB-Hard} & \multicolumn{2}{c}{Texture} & \multicolumn{2}{c}{Places}  & \multicolumn{2}{c}{Avg.} \\ \midrule Method&\texttt{COME} & Acc$\uparrow$ &FPR$\downarrow$ & Acc$\uparrow$ &FPR$\downarrow$ & Acc$\uparrow$ &FPR$\downarrow$ & Acc$\uparrow$ &FPR$\downarrow$ & Acc$\uparrow$ &FPR$\downarrow$ & Acc$\uparrow$ &FPR$\downarrow$ & Acc$\uparrow$ &FPR$\downarrow$ \\
      \midrule 
No Adapt
& \xmark &64.4&63.7&64.5&69.9&64.4&69.5&64.4&72.5&64.8&65.6&64.3&56.8&64.4&66.3\\
\midrule 
PL
& \xmark &69.1&62.8&65.6&71.6&68.8&69.5&68.4&75.9&66.1&66.0&66.6&59.7&67.4&67.6\\

T3A
& \xmark &64.4&71.2&64.3&70.0&64.2&75.0&63.7&80.7&64.4&69.0&63.8&69.5&64.1&72.6\\

TEA
& \xmark &63.9&63.8&60.5&72.5&62.3&74.6&63.3&79.4&61.0&67.8&61.9&64.5&62.2&70.4\\

LAME
& \xmark &64.1&64.4&64.1&72.3&64.1&72.4&64.2&74.0&64.7&68.8&64.0&61.3&64.2&68.9\\

FOA
& \xmark &67.8&61.2&66.4&70.5&67.4&66.1&67.1&75.6&65.8&61.4&66.6&54.1&66.8&64.8\\

\midrule \multirow{3}{*}{Tent}
&\xmark   &70.8&63.2&66.2&71.9&69.9&70.7&69.8&77.4&66.4&66.5&68.3&59.9&68.6&68.3\\
&\cmark  &72.6&64.7&68.9&64.3&72.5&63.7&72.7&70.7&68.4&60.4&70.7&45.9&71.0&61.6\\
&Improve  &\textcolor{mycolor4}{$\bigtriangleup1.7$} &\textcolor{mycolor1}{$\bigtriangleup1.6$} &\textcolor{mycolor4}{$\bigtriangleup2.7$} &\textcolor{mycolor4}{$\bigtriangledown7.6$} &\textcolor{mycolor4}{$\bigtriangleup2.6$} &\textcolor{mycolor4}{$\bigtriangledown7.0$} &\textcolor{mycolor4}{$\bigtriangleup2.9$} &\textcolor{mycolor4}{$\bigtriangledown6.7$} &\textcolor{mycolor4}{$\bigtriangleup2.1$} &\textcolor{mycolor4}{$\bigtriangledown6.1$} &\textcolor{mycolor4}{$\bigtriangleup2.4$} &\textcolor{mycolor4}{$\bigtriangledown14.0$} &\cellcolor{gray!20}\textcolor{mycolor4}{$\bigtriangleup2.4$} &\cellcolor{gray!20}\textcolor{mycolor4}{$\bigtriangledown6.6$}\\

\midrule \multirow{3}{*}{EATA}
&\xmark   &70.3&63.7&66.4&68.6&70.3&71.5&70.0&77.4&67.3&67.1&68.8&61.9&68.8&68.4\\
&\cmark  &73.4&62.7&70.1&60.5&73.2&63.3&73.0&70.5&70.5&55.8&72.3&45.6&72.1&59.7\\
&Improve  &\textcolor{mycolor4}{$\bigtriangleup3.1$} &\textcolor{mycolor4}{$\bigtriangledown1.0$} &\textcolor{mycolor4}{$\bigtriangleup3.7$} &\textcolor{mycolor4}{$\bigtriangledown8.2$} &\textcolor{mycolor4}{$\bigtriangleup2.9$} &\textcolor{mycolor4}{$\bigtriangledown8.2$} &\textcolor{mycolor4}{$\bigtriangleup3.0$} &\textcolor{mycolor4}{$\bigtriangledown6.9$} &\textcolor{mycolor4}{$\bigtriangleup3.2$} &\textcolor{mycolor4}{$\bigtriangledown11.3$} &\textcolor{mycolor4}{$\bigtriangleup3.6$} &\textcolor{mycolor4}{$\bigtriangledown16.3$} &\cellcolor{gray!20}\textcolor{mycolor4}{$\bigtriangleup3.3$} &\cellcolor{gray!20}\textcolor{mycolor4}{$\bigtriangledown8.6$}\\

\midrule \multirow{3}{*}{SAR}
&\xmark   &69.7&62.3&64.9&71.4&66.9&70.9&67.7&78.1&64.4&64.5&66.1&58.6&66.6&67.6\\
&\cmark  &73.1&62.9&69.8&66.3&73.2&65.2&73.5&71.8&69.5&59.4&72.3&49.8&71.9&62.6\\
&Improve  &\textcolor{mycolor4}{$\bigtriangleup3.5$} &\textcolor{mycolor1}{$\bigtriangleup0.6$} &\textcolor{mycolor4}{$\bigtriangleup4.9$} &\textcolor{mycolor4}{$\bigtriangledown5.1$} &\textcolor{mycolor4}{$\bigtriangleup6.3$} &\textcolor{mycolor4}{$\bigtriangledown5.6$} &\textcolor{mycolor4}{$\bigtriangleup5.9$} &\textcolor{mycolor4}{$\bigtriangledown6.3$} &\textcolor{mycolor4}{$\bigtriangleup5.1$} &\textcolor{mycolor4}{$\bigtriangledown5.1$} &\textcolor{mycolor4}{$\bigtriangleup6.3$} &\textcolor{mycolor4}{$\bigtriangledown8.8$} &\cellcolor{gray!20}\textcolor{mycolor4}{$\bigtriangleup5.3$} &\cellcolor{gray!20}\textcolor{mycolor4}{$\bigtriangledown5.1$}\\

\midrule \multirow{3}{*}{CoTTA}
&\xmark   &67.6&63.4&65.3&69.7&70.4&69.5&70.3&76.0&65.8&66.2&66.6&59.2&67.6&67.3\\
&\cmark  &70.5&62.5&66.2&68.8&72.4&73.5&72.2&78.7&66.5&64.7&68.9&55.0&69.4&67.2\\
&Improve  &\textcolor{mycolor4}{$\bigtriangleup2.9$} &\textcolor{mycolor4}{$\bigtriangledown0.9$} &\textcolor{mycolor4}{$\bigtriangleup0.9$} &\textcolor{mycolor4}{$\bigtriangledown0.9$} &\textcolor{mycolor4}{$\bigtriangleup2.0$} &\textcolor{mycolor1}{$\bigtriangleup4.0$} &\textcolor{mycolor4}{$\bigtriangleup2.0$} &\textcolor{mycolor1}{$\bigtriangleup2.7$} &\textcolor{mycolor4}{$\bigtriangleup0.7$} &\textcolor{mycolor4}{$\bigtriangledown1.5$} &\textcolor{mycolor4}{$\bigtriangleup2.3$} &\textcolor{mycolor4}{$\bigtriangledown4.2$} &\cellcolor{gray!20}\textcolor{mycolor4}{$\bigtriangleup1.8$} &\cellcolor{gray!20} - \\

\midrule \multirow{3}{*}{MEMO}
&\xmark   &64.8&69.8&64.8&77.5&64.7&71.9&64.8&77.3&65.0&79.3&64.6&71.5&64.8&74.5\\
&\cmark  &65.2&67.8&65.9&76.4&65.2&70.8&65.3&75.0&65.4&76.7&65.3&67.6&65.4&72.4\\
&Improve  &\textcolor{mycolor4}{$\bigtriangleup0.5$} &\textcolor{mycolor4}{$\bigtriangledown1.9$} &\textcolor{mycolor4}{$\bigtriangleup1.1$} &\textcolor{mycolor4}{$\bigtriangledown1.1$} &\textcolor{mycolor4}{$\bigtriangleup0.5$} &\textcolor{mycolor4}{$\bigtriangledown1.1$} &\textcolor{mycolor4}{$\bigtriangleup0.5$} &\textcolor{mycolor4}{$\bigtriangledown2.3$} & - &\textcolor{mycolor4}{$\bigtriangledown2.6$} &\textcolor{mycolor4}{$\bigtriangleup0.7$} &\textcolor{mycolor4}{$\bigtriangledown3.9$} &\cellcolor{gray!20}\textcolor{mycolor4}{$\bigtriangleup0.6$} &\cellcolor{gray!20}\textcolor{mycolor4}{$\bigtriangledown2.2$}\\

    \bottomrule
    \end{tabular}
    }
    \vspace{-10pt}
    \end{table}

\textbf{Performance comparison in open-world and lifelong TTA settings (Q2).} In Table \ref{tab:openworld} and \ref{tab:lifelong}, we present the results under open-world and lifelong TTA settings respectively. In open-world TTA, the test data distribution is a mixture of both normal covariate-shifted data and abnormal outliers. The mixture ratio of $P^{\rm Cov}$ and $P^{\rm Outlier}$ is 0.5 following previous work~\citep{bai2023feed}, i.e., $P^{\rm test}=0.5 P^{\rm Cov}+0.5 P^{\rm Outlier}$. Such outliers arise from unknown classes that are not present in training data, which should not be classified into any class $k\in\mathcal{Y}$ for model trustworthiness. According to the experimental results, it is observed that our \texttt{COME} can consistently improve the performance of existing TTA methods.

\textbf{Reliability of uncertainty estimation (Q3).} We visualize the distribution of model confidence, i.e., the maximum predicted class probability\footnote{For our COME, the class probability is calculated according to Eq. 5.} in open-world TTA setting, where the covariate-shifted samples is ImageNet-C (Gaussian noise level 3), and outliers are Ninco. As shown in Figure~\ref{fig:uncertainty}, the model confidence of our \texttt{COME} can effectively perceive incorrect predictions, which establishes an distinguishable margin. In contrast to the model confidence of EM which is almost identical for correct-classified samples, mis-classified samples and outliers, the model confidence of our method can provide more meaningful information with which to differentiate them.

\begin{table}[htbp]
    \centering
    \caption{Classification and uncertainty estimation comparisons under \textbf{lifelong} TTA settings. The model is online adapted and the parameters will never be reset, yet the test input distribution might exhibit a continual shift over time. Performance on each individual corruptions are in Appendix \ref{lifelong-fprs}}
    \label{tab:lifelong}
\resizebox{\textwidth}{!}{
\setlength{\tabcolsep}{0.75mm}
\begin{tabular}{cc|ccc|cccc|cccc|cccc|cc}
\toprule
 & &  \multicolumn{3}{c}{Noise} & \multicolumn{4}{c}{Blur} & \multicolumn{4}{c}{Weather} & \multicolumn{4}{c}{Digital} & \multicolumn{2}{c}{Avg.}  \\ 
\cmidrule(r){3-5} \cmidrule(r){6-9} \cmidrule(r){10-13} \cmidrule(r){14-17} \cmidrule(r){18-19}  
 Methods & \texttt{COME} & Gauss. & Shot & Impul. & Defoc & Glass & Motion & Zoom & Snow & Frost &Fog & Brit. &Contr. &Elast. &Pixel & JPEG & Acc$\uparrow$ &FPR$\downarrow$\\
 \toprule
No Adapt
&\xmark   &35.1 & 32.2 & 35.9 & 31.4 & 25.3 & 39.4 & 31.6 & 24.5 & 30.1 & 54.7 & 64.5 & 49.0 & 34.2 & 53.2 & 56.5 & 39.8&68.4\\
\midrule
PL
&\xmark   &49.9 & 53.4 & 56.7 & 46.7 & 46.1 & 56.6 & 51.4 & 52.6 & 60.2 & 68.1 & 77.7 & 64.5 & 51.0 & 69.0 & 68.8 & 58.2&69.0\\

FOA
&\xmark   &46.5 & 47.0 & 51.1 & 44.8 & 45.5 & 52.3 & 48.1 & 49.7 & 57.9 & 68.0 & 76.4 & 63.6 & 51.7 & 62.2 & 65.3 & 55.3&64.8\\

\midrule \multirow{3}{*}{Tent}
&\xmark   &52.4 & 56.2 & 58.7 & 50.9 & 51.1 & 57.7 & 52.7 & 54.7 & 60.5 & 68.4 & 77.3 & 64.6 & 53.8 & 69.3 & 68.6 & 59.8&71.8\\
&\cmark  &54.7 & 59.0 & 60.3 & 51.2 & 53.7 & 60.6 & 57.4 & 64.1 & 65.8 & 71.0 & 78.6 & 66.9 & 62.7 & 71.8 & 70.5 & 63.2&67.1\\
&Improve  &\textcolor{mycolor4}{$\bigtriangleup2.3$} & \textcolor{mycolor4}{$\bigtriangleup2.8$} & \textcolor{mycolor4}{$\bigtriangleup1.6$} & - & \textcolor{mycolor4}{$\bigtriangleup2.6$} & \textcolor{mycolor4}{$\bigtriangleup2.9$} & \textcolor{mycolor4}{$\bigtriangleup4.7$} & \textcolor{mycolor4}{$\bigtriangleup9.5$} & \textcolor{mycolor4}{$\bigtriangleup5.3$} & \textcolor{mycolor4}{$\bigtriangleup2.6$} & \textcolor{mycolor4}{$\bigtriangleup1.3$} & \textcolor{mycolor4}{$\bigtriangleup2.3$} & \textcolor{mycolor4}{$\bigtriangleup9.0$} & \textcolor{mycolor4}{$\bigtriangleup2.5$} & \textcolor{mycolor4}{$\bigtriangleup1.9$} &\cellcolor{gray!20} \textcolor{mycolor4}{$\bigtriangleup3.4$}&\cellcolor{gray!20}\textcolor{mycolor4}{$\bigtriangledown4.7$}\\
\midrule \multirow{3}{*}{EATA}
&\xmark   &55.9 & 59.5 & 60.9 & 56.2 & 59.2 & 63.0 & 61.6 & 65.7 & 67.5 & 72.5 & 78.6 & 66.8 & 67.1 & 72.2 & 71.7 & 65.2&70.3\\
&\cmark  &57.9 & 60.6 & 61.5 & 57.0 & 59.8 & 63.7 & 62.3 & 67.2 & 68.3 & 73.7 & 78.8 & 69.7 & 68.6 & 73.1 & 71.8 & 66.3&66.7\\
&Improve  &\textcolor{mycolor4}{$\bigtriangleup2.0$} & \textcolor{mycolor4}{$\bigtriangleup1.0$} & \textcolor{mycolor4}{$\bigtriangleup0.6$} & \textcolor{mycolor4}{$\bigtriangleup0.8$} & \textcolor{mycolor4}{$\bigtriangleup0.6$} & \textcolor{mycolor4}{$\bigtriangleup0.8$} & \textcolor{mycolor4}{$\bigtriangleup0.7$} & \textcolor{mycolor4}{$\bigtriangleup1.5$} & \textcolor{mycolor4}{$\bigtriangleup0.8$} & \textcolor{mycolor4}{$\bigtriangleup1.2$} & - & \textcolor{mycolor4}{$\bigtriangleup2.9$} & \textcolor{mycolor4}{$\bigtriangleup1.4$} & \textcolor{mycolor4}{$\bigtriangleup0.9$} & - & \cellcolor{gray!20}\textcolor{mycolor4}{$\bigtriangleup1.0$}&\cellcolor{gray!20}\textcolor{mycolor4}{$\bigtriangledown3.6$}\\
\midrule \multirow{3}{*}{SAR}
&\xmark   &52.0 & 54.8 & 56.2 & 50.2 & 52.3 & 56.1 & 52.8 & 50.8 & 26.5 & 0.1 & 3.1 & 0.1 & 0.1 & 0.1 & 0.3 & 30.4&80.8\\
&\cmark  &56.0 & 60.0 & 61.1 & 56.4 & 58.1 & 62.9 & 60.4 & 66.1 & 67.4 & 72.2 & 78.7 & 68.0 & 66.0 & 72.6 & 70.9 & 65.1&65.7\\
&Improve  &\textcolor{mycolor4}{$\bigtriangleup4.0$} & \textcolor{mycolor4}{$\bigtriangleup5.3$} & \textcolor{mycolor4}{$\bigtriangleup4.9$} & \textcolor{mycolor4}{$\bigtriangleup6.2$} & \textcolor{mycolor4}{$\bigtriangleup5.8$} & \textcolor{mycolor4}{$\bigtriangleup6.8$} & \textcolor{mycolor4}{$\bigtriangleup7.6$} & \textcolor{mycolor4}{$\bigtriangleup15.3$} & \textcolor{mycolor4}{$\bigtriangleup41.0$} & \textcolor{mycolor4}{$\bigtriangleup72.1$} & \textcolor{mycolor4}{$\bigtriangleup75.6$} & \textcolor{mycolor4}{$\bigtriangleup67.9$} & \textcolor{mycolor4}{$\bigtriangleup65.9$} & \textcolor{mycolor4}{$\bigtriangleup72.5$} & \textcolor{mycolor4}{$\bigtriangleup70.6$} &\cellcolor{gray!20} \textcolor{mycolor4}{$\bigtriangleup34.8$}&\cellcolor{gray!20}\textcolor{mycolor4}{$\bigtriangledown15.1$}\\
\midrule \multirow{3}{*}{COTTA}
&\xmark   &40.3 & 49.2 & 57.1 & 39.8 & 50.4 & 55.6 & 48.3 & 53.1 & 61.3 & 63.9 & 73.3 & 62.0 & 56.5 & 67.5 & 66.7 & 56.3&69.9\\
&\cmark  &49.7 & 61.7 & 64.2 & 45.7 & 57.0 & 59.0 & 51.1 & 58.2 & 63.1 & 66.0 & 73.4 & 62.9 & 58.0 & 68.9 & 68.2 & 60.5&65.7\\
&Improve  &\textcolor{mycolor4}{$\bigtriangleup9.4$} & \textcolor{mycolor4}{$\bigtriangleup12.4$} & \textcolor{mycolor4}{$\bigtriangleup7.1$} & \textcolor{mycolor4}{$\bigtriangleup5.9$} & \textcolor{mycolor4}{$\bigtriangleup6.6$} & \textcolor{mycolor4}{$\bigtriangleup3.4$} & \textcolor{mycolor4}{$\bigtriangleup2.8$} & \textcolor{mycolor4}{$\bigtriangleup5.1$} & \textcolor{mycolor4}{$\bigtriangleup1.8$} & \textcolor{mycolor4}{$\bigtriangleup2.1$} & - & \textcolor{mycolor4}{$\bigtriangleup1.0$} & \textcolor{mycolor4}{$\bigtriangleup1.5$} & \textcolor{mycolor4}{$\bigtriangleup1.3$} & \textcolor{mycolor4}{$\bigtriangleup1.5$} &\cellcolor{gray!20} \textcolor{mycolor4}{$\bigtriangleup4.1$}&\cellcolor{gray!20}\textcolor{mycolor4}{$\bigtriangledown4.2$}\\

\bottomrule
\end{tabular}
}
\end{table}

\textbf{Ablation study.} Finally, we conduct the ablation study on different components in our \texttt{COME}, i.e., with and without the uncertainty constraint in Eq.~\ref{constrained-op}. The experimental results are shown in Table \ref{tab:ablation study}, where SL indicates minimizing entropy of the subjective logic opinion and UC means the uncertainty constraint described by Eq.~\ref{constrained-op}. Compared with non-constrained optimization, naively minimizing the entropy of subjective opinion can only slightly improve uncertainty estimation performance. Combining with the uncertainty constraint, the average and worst-case accuracy can be both substantially improved, which indicates the optimal design of our \texttt{COME}.

\begin{figure}[!htbp]
  \begin{minipage}[t]{0.62\textwidth}
    \centering
    \includegraphics[width=\linewidth]{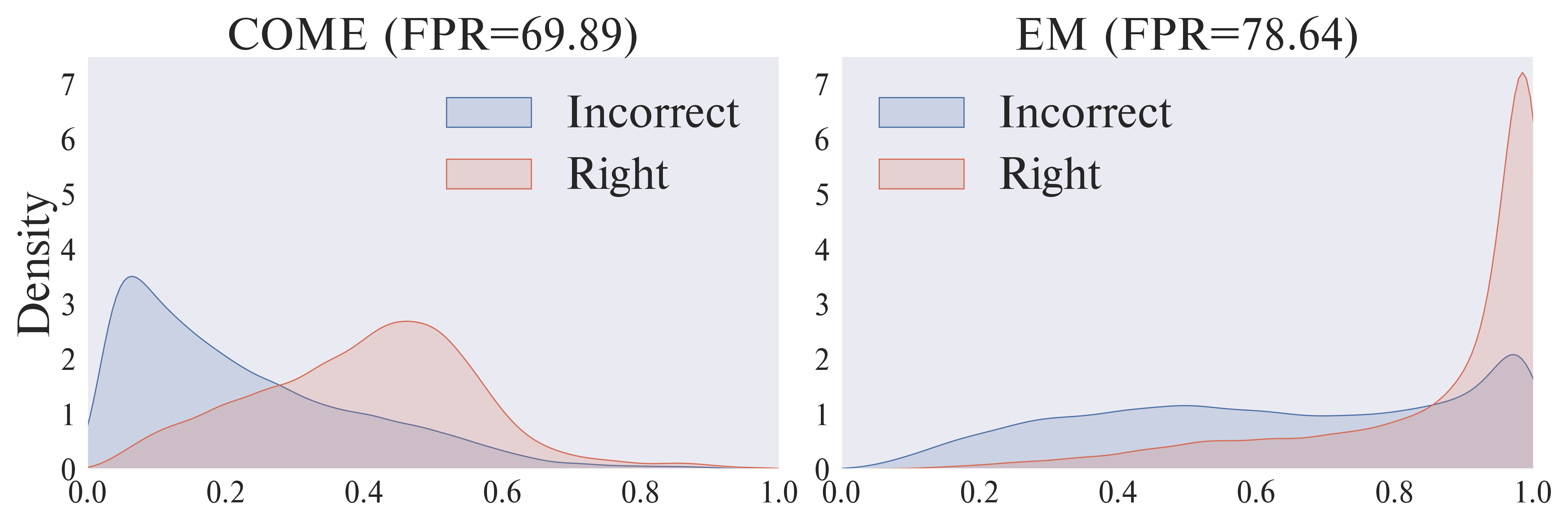}
    \vspace{-20pt}
    \captionof{figure}{Distribution of model confidence.}
    \label{fig:uncertainty}
  \end{minipage}
  \hfill
  \begin{minipage}[t]{0.38\textwidth}
  \setlength{\tabcolsep}{1.5mm}
    \centering
    \vspace{-2.6cm}
    \resizebox{\textwidth}{!}{
    \begin{tabular}{cc|cc|cc}
    
    \toprule   \multicolumn{2}{c}{} & \multicolumn{2}{c}{Acc$\uparrow$}  & \multicolumn{2}{c}{FPR$\downarrow$}  \\
       SL & UC & Mean &Worst & Mean &Worst \\

\midrule
\xmark & \xmark & 52.8 & 10.6 & 70.2 & 95.3 \\
\cmark & \xmark & 52.7 & 10.4 & 70.0 & 94.9 \\
\xmark & \cmark & 60.9 & 25.5 & 68.0 & 93.0 \\
\cmark & \cmark &\cellcolor{gray!20} 61.2 &\cellcolor{gray!20} 51.7 &\cellcolor{gray!20} 67.3 &\cellcolor{gray!20} 71.4 \\
    \bottomrule
    \end{tabular}
    }
    \captionof{table}{Ablation study.}
    \label{tab:ablation study}
  \end{minipage}
\end{figure}

\section{Conclusion}
In this paper, we propose a novel learning principle called \texttt{COME} to improve existing TTA methods. Our \texttt{COME} explicitly models the uncertainty raising upon unreliable test samples using the theory of evidence, and then regularizes the model to in favor of conservative prediction confidence during inference time. Our method takes inspiration from Bayesian framework, and consistently outperforms previous EM-based TTA methods on commonly-used benchmarks.


\appendix

\clearpage 
\newpage
\appendix

\hypersetup{ colorlinks = false, pdfborder = {0 0 0},
    linkbordercolor = {1 1 1} }

\startcontents[sections]
\printcontents[sections]{l}{1}{\section*{Appendices}\setcounter{tocdepth}{2}}

\section{Proofs}
To proof Lemma 1 and Theorem 1, we need the following lemma firstly.
\begin{lemma}
    Let $p, q$ be two real numbers. Assuming that $p\leq q$, then the $p$-norm (also called $\ell^p$-norm) and $q$-norm of vector $x=(x_1,\cdots,x_n)$ satisfied
    \begin{equation}
        ||x||_p\leq n^{(1/q-1/p)}||x||_q.
    \end{equation}
    where $n$ is the length of the vector.
\end{lemma}
\begin{proof} Recall H\"older's inequality
\begin{equation}
    \sum_{i=1}^n |a_i||b_i|\leq (\sum_{i=1}^n |a_i|^r)^{1/r}(\sum_{i=1}^n |b_i|^{\frac{r}{r-1}})^{1-\frac{1}{r}}.
\end{equation}
Apply this inequality to the case that $|a_i|=|x_i|^p$, $|b_i|=1$ and $r=q/p\geq 1$, we can derive to
\begin{equation}
    \sum_{i=1}^n |a_i||b_i|\leq (\sum_{i=1}^n((x_i)^p)^{\frac{q}{p}})^{\frac{p}{q}}(\sum_{i=1}^n 1^{\frac{q}{q-p}})^{1-{\frac{p}{q}}}=(\sum^{n}_{i=1}|x_i|^q)^{\frac{p}{q}}n^{1-\frac{p}{q}}.
\end{equation}
Then we have
\begin{equation}
\begin{split}
    ||x||_p&=(\sum_{i=1}^n |x_i|_p)^{1/p}\\
   & \leq ((\sum_{i=1}^n |x_i|^q)^{\frac{p}{q}} n^{1-\frac{p}{q}})^{1/p}\\
   &=(\sum_{i=1}^n|x_i|^q)^{\frac{1}{q}} n^{\frac{1}{p}-\frac{1}{q}}\\
   &=n^{1/p-1/q}||x||_q
\end{split}
\end{equation}
\end{proof}
Now we proceed to proof our main results.

\begin{proof}{{Proof of Lemma 1.}} Let $f_{\max}=\max\limits_{i}\ f_i(x)$, then we have
\begin{equation}
    \exp(f_{\max})\leq \sum_{i=1}^K \exp (f_i)\leq K \exp(f_{\max}),
\end{equation}
Applying the logarithm to the inequality, then
\begin{equation}
    f_{\max}\leq {\rm LSE}(f)\leq f_{\max}+\log K,
\end{equation}
where $\rm LSE$ is the shorthand of LogSumExp function, i.e., ${\rm LSE}(x):=\log\sum\limits_{i=1}^K \exp x_i$. 

Since we assume that all the elements in logits $x$ are all positive, then $f_{\max}=||f||_{\infty}$. Thus combining with Lemma 1 we can derive that
\begin{equation}
    K^{-1/p}||f||_p\leq{\rm LSE}(x)\leq ||f||_p+\log K,
\end{equation}
Noted that $u=K/{\rm LSE}(f)$, then we have
\begin{equation}
   \frac{K}{||f||_p+\log K}\leq u\leq\frac{ K^{1+1/p}}{||f||_p}.
\end{equation}
\end{proof}

\begin{proof}{Proof of Theorem 1.} Assuming that the uncertainty mass $u$ is constrained as

\begin{equation}
    u_0-\delta\leq u\leq u_0+\delta,
\end{equation}
then the ${\rm LSE}$ function of model output is also bounded by
\begin{equation}
    \frac{K}{u_0+\delta}\leq{\rm LSE}(f(x))\leq \frac{K}{u_0-\delta}.
\end{equation}
Noted that
\begin{equation}
    \max f(x)\leq {\rm LSE}(f(x)),
\end{equation}
and thus
\begin{equation}
    \max f(x)\leq \frac{K}{u_0-\delta}.
\end{equation}
According to Eq. 5, the model confidence is calculated by
\begin{equation}
    \max_k \mu_{k}(x)=\frac{\alpha_k}{S}=\frac{\exp f_{\max}}{\sum_{i=1}^K \exp f_i},
\end{equation}
where $\alpha_i =\exp f_i(x)$. 

Assuming the $f_j$ is the largest element in $f(x)$, then
\begin{equation}
    \begin{split}
        \max_k \mu_{k}(x)&=\frac{1}{1+\sum_{i=1,i\neq j}^K \exp {(f_i-f_{\max})}}
        \\
    &\leq \frac{1}{1+(K-1)\exp{(f_{\min}-f_{\max})}}
    \\
    &\leq \frac{1}{1+(K-1)\exp{(-\frac{K}{u_0-\delta})}}
    \end{split}
\end{equation}
\end{proof}

\section{Experimental details}

\subsection{Datasets}

\textbf{Covariate-shifted OOD generalization datasets.} We conduct experiments on ImageNet-C~\citep{hendrycks2018benchmarking}, which consists of 15 types of algorithmically generated corruptions from noise, blur, weather, and digital categories. Each type of corruption has 5 levels of severity, resulting in 75 distinct corruptions.

\textbf{Abnormal outliers for open-world TTA experiments.}
We follow the settings of~\citep{zhang2023openood}, where OpenImage-O~\citep{wang2022vim}, SSB-hard~\citep{vaze2021open}, Textures~\citep{cimpoi2014describing}, iNaturalist~\citep{van2018inaturalist} and NINCO~\citep{bitterwolf2023ninco} are selected as outliers for ImageNet. $\circ$ OpenImage-O contains 17632 manually filtered images and is 7.8 \( \times \) larger than the ImageNet dataset. $\circ$ SSB-hard is selected from ImageNet-21K. It consists of 49K images and 980 categories. $\circ$ Textures (Describable Textures Dataset, DTD) consists of 5,640 images depicting natural textures. $\circ$ iNaturalist consists of 859000 images from over 5000 different species of plants and animals. $\circ$ NINCO consists with a total of 5879 samples of 64 classes which are non-overlapped with ImageNet-C.

\subsection{Implementation details}
\textbf{Pretrained models.} The pretrained ViT model is ViT-Base~\citep{dosovitskiy2020vit}. The model is trained on the source ImageNet-1K training set and the model weights\footnote{https://huggingface.co/google/vit-base-patch16-224} are directly obtained from timm respository. Specifically, the pretrained ResNet model in  ~\ref{COMPARISON ON RESNET-50} is ResNet-50-BN~\citep{he2016deep}. The model is trained on the source ImageNet-1K training set and the model weights\footnote{https://download.pytorch.org/models/resnet50-19c8e357.pth} are directly obtained from torchvision library.

\textbf{TEA\footnote{https://github.com/yuanyige/tea}~\citep{yuan2023energy}} We follow all hyperparameters that are set in TEA unless it does not provide. Specifically, we use SGD as the update rule, with a momentum of 0.9, batch size of 64 and learning rate of 0.001/0.00025 for ViT/ResNet models. The trainable parameters are all affine parameters of layer/batch normalization layers for ViT/ResNet models.

\textbf{T3A\footnote{https://github.com/matsuolab/T3A}~\citep{iwasawa2021test}} We follow all hyperparameters that are set in T3A unless it does not provide. Specifically, the batch size is set to 64. The number of supports to restore M is set to 20 for all experiments.

\textbf{LAME\footnote{https://github.com/fiveai/LAME}~\citep{Boudiaf_2022_CVPR}} We follow all hyperparameters that are set in LAME unless it does not provide. For fair comparison, we maintain a consistent batch size of 64 for LAME, aligning it with the same batch size used by other methods in our evaluation. We use the kNN affinity matrix with the value of $k=5$.

\textbf{FOA\footnote{https://github.com/mr-eggplant/FOA}~\citep{niu2024test}} We follow all hyperparameters that are set in FOA unless it does not provide. Specifically, the batch size is set to 64. The number of supports to restore M is set to 20 for all experiments.

\textbf{Tent\footnote{https://github.com/DequanWang/tent}~\citep{wang2020tent}} We follow all hyperparameters that are set in Tent unless it does not provide. Specifically, we use SGD as the update rule, with a momentum of 0.9, batch size of 64 and learning rate of 0.001/0.00025 for ViT/ResNet models. The trainable parameters are all affine parameters of layer/batch normalization layers for ViT/ResNet models. 

\textbf{EATA\footnote{https://github.com/mr-eggplant/EATA}~\citep{niu2022efficient}} We follow all hyperparameters that are set in EATA. Specifically, the entropy constant $E_0$ (for reliable sample identification) is set to $0.4\times ln1000$, where $1000$ is the number of task classes. The $\epsilon$ for redundant sample identification is set to 0.05. The trade-off parameter $\beta$ for entropy loss and regularization loss is set to 2,000. The number of pre-collected in-distribution test samples for Fisher importance calculation is 2,000. We use SGD as the update rule, with a momentum of 0.9, batch size of 64 and learning rate of 0.001/0.00025 for ViT/ResNet models. The trainable parameters are all affine parameters of layer/batch normalization layers for ViT/ResNet models.

\textbf{SAR\footnote{https://github.com/mr-eggplant/SAR}~\citep{niu2023towards}} We follow all hyperparameters that are set in SAR. Specifically, the entropy threshold $E_0$ is set to $0.4\times ln1000$, where $1000$ is the number of task classes. We use SGD as the update rule, with a momentum of 0.9, batch size of 64 and learning rate of 0.001/0.00025 for ViT/ResNet models. For model recovery, we follow all strategy that are set in SAR(except for the experiments of life-long).The trainable parameters are all affine parameters of layer/batch normalization layers for ViT/ResNet models. 

\textbf{CoTTA\footnote{https://github.com/qinenergy/cotta}~\citep{wang2022continual}} We follow all hyperparameters that are set in CoTTA unless it does not provide. Specifically, we use SGD as the update rule, with a momentum of 0.9, batch size of 64 and learning rate of 0.001/0.01 for ViT/ResNet models. The augmentation threshold $p_{th}$ is set to 0.1. For images below threshold, we conduct 32 augmentations including color jitter, random affine, Gaussian blur, random horizonal flip, and Gaussian noise. The restoration probability of is set to 0.01 and the EMA factor $\alpha$ for teacher update is set to 0.999. The trainable parameters are all affine parameters of layer/batch normalization layers for ViT/ResNet models. 

\textbf{MEMO\footnote{https://github.com/zhangmarvin/memo}~\citep{zhang2022memo}} We follow all hyperparameters that are set in MEMO. Specifically, we use the AugMix\footnote{https://github.com/google-research/augmix}~\citep{hendrycks2020augmix} as a set of data augmentations and the augmentation size is set to 32. We use SGD as the optimizer,with learning rate 0.00025 and no weight decay. The trainable parameters are the entire model.

\section{Additional results}
\subsection{In-Distribution Performance}
We compare the In-Distribution Performance of proposed \texttt{COME} to EM-based methods introduced by~\citep{niu2024test}. As shown in Table \ref{tab:IN-DISTRIBUTION PERFORMANCE}, our method consistently improves the performance of all EM methods on classification tasks.


\begin{table}[htbp]
\vspace{-5pt}
    \centering
    \caption{Comparison w.r.t. \textbf{in-distribution performance}, $i.e.$, on clean/original ImageNet validation set, with ViT as the base model. Substantial ($\geq 0.5$) \textcolor{mycolor4}{improvement} and \textcolor{mycolor3}{degradation} compared to the baseline are highlighted in blue or red respectively.}
    \label{tab:IN-DISTRIBUTION PERFORMANCE}
    \resizebox{0.7\textwidth}{!}{
    \setlength{\tabcolsep}{2.5mm}
    \begin{tabular}{c|c|c|c|c|c|c}
    \toprule
       & TENT& SAR &EATA &CoTTA&MEMO   &Avg. \\ \midrule\texttt{COME} & Acc$\uparrow$  & Acc$\uparrow$  & Acc $\uparrow$ & Acc $\uparrow$ & Acc$\uparrow$  & Acc$\uparrow$  \\

\midrule
\xmark   &81.4&80.7&81.3&82.1&80.3&81.2\\
\cmark  &83.1&83.1&83.1&82.8&80.6&82.6\\
Improve  &\textcolor{mycolor4}{$\bigtriangleup1.7$} &\textcolor{mycolor4}{$\bigtriangleup2.5$} &\textcolor{mycolor4}{$\bigtriangleup1.8$} &\textcolor{mycolor4}{$\bigtriangleup0.7$} &-&\textcolor{mycolor4}{$\bigtriangleup1.4$} \\

    \bottomrule
    \end{tabular}
    }\vspace{-5pt}
    \end{table}
    
\subsection{Comparison on ResNet-50}
\label{COMPARISON ON RESNET-50}
As shown in previous work~\citep{niu2023towards}, the TTA performance can varying a lot from different model architecture, especially various types of normalization layers, i.e., batch normalization, group normalization, layer normalization and instance normalization. Thus we conduct additional experiments to further evaluate the proposed method on resnet-50 with batch normalization layers under open-world TTA settings. The experimental results in Table~\ref{tab:COMPARISON ON RESNET-50} show that our \texttt{COME} still achieves superior performance compared to EM learning principle when equipped to Tent.
\begin{table}[htbp]
\vspace{-5pt}
\centering
    \caption{Classification and uncertainty estimation comparisons under \textbf{open-world} TTA settings with \textbf{ResNet-50-BN}, where $P^{\rm test}=0.5 P^{\rm Cov}+0.5 P^{\rm Sem}$(Gaussian noise of severity level 3) and a suit of diverse abnormal outliers as same with Table \ref{tab:openworld}. Substantial ($\geq 0.5$) \textcolor{mycolor4}{improvement} and \textcolor{mycolor3}{degradation} compared to the baseline are highlighted in blue or red respectively.}
    \label{tab:COMPARISON ON RESNET-50}
    \resizebox{\textwidth}{!}{
    \setlength{\tabcolsep}{1.5mm}
    \begin{tabular}{cc|cc|cc|cc|cc|cc|cc|cc}
    \toprule
      & & \multicolumn{2}{c}{None} & \multicolumn{2}{c}{NINCO} & \multicolumn{2}{c}{iNaturist} & \multicolumn{2}{c}{SSB-Hard} & \multicolumn{2}{c}{Texture} & \multicolumn{2}{c}{Places}  & \multicolumn{2}{c}{Avg.} \\ \midrule Method&\texttt{COME} 
      & Acc$\uparrow$ &FPR$\downarrow$ & Acc$\uparrow$ &FPR$\downarrow$ & Acc$\uparrow$ &FPR$\downarrow$ & Acc$\uparrow$ &FPR$\downarrow$ & Acc$\uparrow$ &FPR$\downarrow$ & Acc$\uparrow$ &FPR$\downarrow$ & Acc$\uparrow$ &FPR$\downarrow$ \\
      \midrule 
No Adapt
& \xmark &3.0&81.6&3.0&91.2&3.0&89.4&3.0&90.2&3.0&88.3&2.8&90.7&3.0&88.6\\
\midrule
PL
& \xmark &26.9&71.2&16.1&84.3&12.9&88.4&15.9&87.8&18.1&86.1&16.9&82.9&17.8&83.5\\

TEA
& \xmark &28.5&73.5&17.6&84.0&9.6&84.5&11.8&88.3&19.9&84.2&16.8&82.9&17.4&82.9\\

\midrule \multirow{3}{*}{Tent}
&\xmark   &52.5&67.5&43.7&79.6&52.1&78.3&51.9&82.1&44.4&78.6&48.7&73.2&48.9&76.6\\
&\cmark  &55.0&67.6&46.3&75.5&54.3&75.1&54.4&81.9&45.8&75.6&50.8&64.0&51.1&73.3\\
&Improve  &\textcolor{mycolor4}{$\bigtriangleup2.6$} & -  &\textcolor{mycolor4}{$\bigtriangleup2.6$} &\textcolor{mycolor4}{$\bigtriangledown4.2$} &\textcolor{mycolor4}{$\bigtriangleup2.2$} &\textcolor{mycolor4}{$\bigtriangledown3.2$} &\textcolor{mycolor4}{$\bigtriangleup2.5$} & -  &\textcolor{mycolor4}{$\bigtriangleup1.4$} &\textcolor{mycolor4}{$\bigtriangledown3.1$} &\textcolor{mycolor4}{$\bigtriangleup2.2$} &\textcolor{mycolor4}{$\bigtriangledown9.2$} &\textcolor{mycolor4}{$\bigtriangleup2.2$} &\textcolor{mycolor4}{$\bigtriangledown3.3$}\\

\midrule \multirow{3}{*}{EATA}
&\xmark   &55.9&68.2&47.8&80.8&53.1&78.4&52.2&82.0&48.7&75.3&49.3&74.5&51.2&76.5\\
&\cmark  &58.0&66.2&52.9&74.8&57.6&73.2&57.4&81.3&52.5&70.7&55.4&62.9&55.6&71.5\\
&Improve  &\textcolor{mycolor4}{$\bigtriangleup2.0$} &\textcolor{mycolor4}{$\bigtriangledown2.0$} &\textcolor{mycolor4}{$\bigtriangleup5.1$} &\textcolor{mycolor4}{$\bigtriangledown6.0$} &\textcolor{mycolor4}{$\bigtriangleup4.5$} &\textcolor{mycolor4}{$\bigtriangledown5.1$} &\textcolor{mycolor4}{$\bigtriangleup5.2$} &\textcolor{mycolor4}{$\bigtriangledown0.7$} &\textcolor{mycolor4}{$\bigtriangleup3.9$} &\textcolor{mycolor4}{$\bigtriangledown4.6$} &\textcolor{mycolor4}{$\bigtriangleup6.0$} &\textcolor{mycolor4}{$\bigtriangledown11.6$} &\textcolor{mycolor4}{$\bigtriangleup4.5$} &\textcolor{mycolor4}{$\bigtriangledown5.0$}\\
\midrule \multirow{3}{*}{SAR}
&\xmark   &51.8&64.6&42.4&78.3&47.6&81.3&48.1&84.4&42.7&79.1&46.0&76.7&46.4&77.4\\
&\cmark  &56.3&64.0&46.7&77.9&55.3&77.1&55.1&81.6&46.4&77.5&52.5&68.1&52.0&74.4\\
&Improve  &\textcolor{mycolor4}{$\bigtriangleup4.5$} &\textcolor{mycolor4}{$\bigtriangledown0.6$} &\textcolor{mycolor4}{$\bigtriangleup4.3$} &\textcolor{mycolor4}{$\bigtriangledown0.3$} &\textcolor{mycolor4}{$\bigtriangleup7.7$} &\textcolor{mycolor4}{$\bigtriangledown4.2$} &\textcolor{mycolor4}{$\bigtriangleup6.9$} &\textcolor{mycolor4}{$\bigtriangledown2.8$} &\textcolor{mycolor4}{$\bigtriangleup3.7$} &\textcolor{mycolor4}{$\bigtriangledown1.6$} &\textcolor{mycolor4}{$\bigtriangleup6.5$} &\textcolor{mycolor4}{$\bigtriangledown8.6$} &\textcolor{mycolor4}{$\bigtriangleup5.6$} &\textcolor{mycolor4}{$\bigtriangledown3.0$}\\

\midrule \multirow{3}{*}{COTTA}
&\xmark   &22.6&70.7&14.4&87.4&21.1&78.6&19.7&84.1&15.5&87.3&15.8&82.0&18.2&81.7\\
&\cmark  &24.5&69.4&14.7&86.1&21.4&81.6&19.4&86.1&16.0&85.7&16.4&82.6&18.7&81.9\\
&Improve  &\textcolor{mycolor4}{$\bigtriangleup1.8$} &\textcolor{mycolor4}{$\bigtriangledown1.3$} & - &\textcolor{mycolor4}{$\bigtriangledown1.3$} & - &\textcolor{mycolor1}{$\bigtriangleup2.9$} & - &\textcolor{mycolor1}{$\bigtriangleup2.0$} &\textcolor{mycolor4}{$\bigtriangleup0.5$} &\textcolor{mycolor4}{$\bigtriangledown1.6$} &\textcolor{mycolor4}{$\bigtriangleup0.6$} &\textcolor{mycolor1}{$\bigtriangleup0.7$} &\textcolor{mycolor4}{$\bigtriangleup0.5$} & - \\

\midrule \multirow{3}{*}{MEMO}
&\xmark   &8.0&83.6&7.5&89.0&7.9&87.9&7.9&89.8&7.8&88.6&7.7&88.4&7.8&87.9\\
&\cmark  &9.1&77.9&8.7&90.2&9.0&87.3&9.1&89.2&9.1&87.4&8.7&88.8&9.0&86.8\\
&Improve  &\textcolor{mycolor4}{$\bigtriangleup1.1$} &\textcolor{mycolor4}{$\bigtriangledown5.7$} &\textcolor{mycolor4}{$\bigtriangleup1.2$} &\textcolor{mycolor1}{$\bigtriangleup1.2$} &\textcolor{mycolor4}{$\bigtriangleup1.2$} &\textcolor{mycolor4}{$\bigtriangledown0.7$} &\textcolor{mycolor4}{$\bigtriangleup1.2$} &\textcolor{mycolor4}{$\bigtriangledown0.6$} &\textcolor{mycolor4}{$\bigtriangleup1.3$} &\textcolor{mycolor4}{$\bigtriangledown1.2$} &\textcolor{mycolor4}{$\bigtriangleup1.1$} & - &\textcolor{mycolor4}{$\bigtriangleup1.2$} &\textcolor{mycolor4}{$\bigtriangledown1.1$}\\

    \bottomrule
    \end{tabular}
    }\vspace{-5pt}
    \end{table}

\subsection{Time-consuming comparison} We compare the time-cost of proposed \texttt{COME} to EM-based methods and Non-EM based methods in Table \ref{tab:TIME AND SIZE}. We run all the experiments on one single NVIDIA 4090 GPU. Our \texttt{COME} does not introduce noticeably extra cost of computation.
\begin{table}[htbp]
\vspace{-5pt}
\centering
    \caption{Comparisons w.r.t. computation complexity. Accuracy (\%) and FPR (\%) are average results on ImageNet-C (level 5) with ViT-Base. The Wall-Clock Time (seconds) and Memory Usage (MB) are measured for processing 50,000 images of ImageNet-C on a single RTX 4090 GPU.}
    \label{tab:TIME AND SIZE}
    \resizebox{0.50\textwidth}{!}{  
    \setlength{\tabcolsep}{1mm}    
    \begin{tabular}{cc|cc|cc}
    \toprule
       Method&\texttt{COME}  &Acc $\uparrow$&FPR $\downarrow$&Memory &Run Time \\
      \midrule 
No Adapt
&\xmark   &39.8&67.5&853&59\\
\midrule 
LAME
&\xmark   &38.7&69.7 &853&62\\
T3A
&\xmark   &41.0&67.7&984&179\\

PL
&\xmark   &51.3&69.1&6393&128\\
FOA
&\xmark   &53.7&63.6 &869&1687\\
TEA
&\xmark   &46.9&68.3&17266&2865\\

\midrule \multirow{3}{*}{Tent}
&\xmark   &52.8&70.1&6393&129\\
&\cmark  &61.2&66.5&6393&130\\
&Improve  &\textcolor{mycolor4}{$\bigtriangleup8.4$}&\textcolor{mycolor4}{$\bigtriangledown3.6$}
&-&-\\
\midrule \multirow{3}{*}{EATA}
&\xmark   &62.1&65.1&6394&135\\
&\cmark  &64.5&63.8&6394&134\\
&Improve  &\textcolor{mycolor4}{$\bigtriangleup2.4$}&\textcolor{mycolor4}{$\bigtriangledown1.3$}&-&-\\
\midrule \multirow{3}{*}{SAR}
&\xmark   &54.2&66.7&6393&253\\
&\cmark  &64.2&63.8&6393&254\\
&Improve  &\textcolor{mycolor4}{$\bigtriangleup10.1$}&\textcolor{mycolor4}{$\bigtriangledown2.9$}
&-&-\\

\midrule \multirow{3}{*}{COTTA}
&\xmark   &46.1&67.9&19612&738\\
&\cmark  &49.1&67.5&19611&739\\
&Improve  &\textcolor{mycolor4}{$\bigtriangleup3.0$}&\textcolor{mycolor4}{$\bigtriangledown0.3$}&-&-\\
\midrule \multirow{3}{*}{MEMO}
&\xmark   &42.3&72.1&5392&20576\\
&\cmark  &43.2&70.8&5392&20530\\
&Improve  &\textcolor{mycolor4}{$\bigtriangleup0.9$}&\textcolor{mycolor4}{$\bigtriangledown1.3$}&-&-\\

    \bottomrule
    \end{tabular}
    }\vspace{-5pt}
    \end{table}
\subsection{Comparison under mixed distributional shifts.} We evaluate the proposed \texttt{COME} in two additional settings introduced by~\citep{niu2023towards}. These scenarios include 1) online imbalanced label distribution shifts, where the test data are sorted by class, and 2) mixed domain shifts, where the test data stream includes several randomly mixed domains with different distribution shifts. As shown in Table \ref{tab:imbalanced label distribution shifts}, our \texttt{COME} consistently outperform EM when equipped with all EM methods. 

As shown in Table \ref{tab:mixed domain shifts}, our method consistently improves the performance of all EM methods on both classification and uncertainty estimation tasks, with the exception of a very slight decrease in the uncertainty estimation task of CoTTA.
\begin{table}[htbp]
\centering
\vspace{-5pt}
\caption{Comparison w.r.t. imbalanced label shifts performance. Results obtained on ViT and ImageNet-C (level 5) under \textbf{imbalanced label shifts} TTA setting, where the imbalance ratio is $\infty$. Substantial ($\geq 0.5$) \textcolor{mycolor4}{improvement} and \textcolor{mycolor3}{degradation} compared to the baseline are highlighted in blue or red respectively.}
\label{tab:imbalanced label distribution shifts}
\resizebox{\textwidth}{!}{
\setlength{\tabcolsep}{1mm}
\begin{tabular}{cc|ccc|cccc|cccc|cccc|c}
\toprule
 & &  \multicolumn{3}{c}{Noise} & \multicolumn{4}{c}{Blur} & \multicolumn{4}{c}{Weather} & \multicolumn{4}{c}{Digital} & \multicolumn{1}{c}{Avg.}  \\ 
\cmidrule(r){3-5} \cmidrule(r){6-9} \cmidrule(r){10-13} \cmidrule(r){14-17} \cmidrule(r){18-18}  
 Methods &\texttt{COME} & Gauss. & Shot & Impul. & Defoc & Glass & Motion & Zoom & Snow & Frost &Fog & Brit. &Contr. &Elast. &Pixel & JPEG & Acc$\uparrow$ \\
 \toprule
No Adapt
&\xmark   &35.1 & 32.2 & 35.9 & 31.4 & 25.3 & 39.4 & 31.6 & 24.5 & 30.1 & 54.7 & 64.5 & 49.0 & 34.2 & 53.2 & 56.5 & 39.8\\
\midrule
PL
&\xmark   &49.7 & 48.6 & 50.9 & 49.8 & 41.5 & 53.0 & 41.9 & 26.6 & 49.0 & 64.3 & 73.6 & 65.6 & 45.2 & 63.9 & 63.0 & 52.4\\

T3A
&\xmark   &33.4 & 30.3 & 34.2 & 31.3 & 26.8 & 38.7 & 32.1 & 25.1 & 29.3 & 54.5 & 62.8 & 48.8 & 37.4 & 51.9 & 56.2 & 39.5\\

TEA
&\xmark   &44.9 & 40.3 & 46.3 & 39.8 & 35.2 & 46.0 & 12.1 & 14.3 & 46.9 & 60.3 & 72.7 & 60.2 & 48.6 & 62.7 & 58.8 & 45.9\\

LAME
&\xmark   &47.0 & 43.3 & 48.2 & 39.8 & 31.8 & 50.3 & 39.4 & 30.5 & 37.1 & 66.0 & 75.4 & 63.5 & 42.0 & 65.1 & 68.1 & 49.8\\

FOA
&\xmark   &41.5 & 39.2 & 43.6 & 42.5 & 33.7 & 45.5 & 41.0 & 44.9 & 44.5 & 60.1 & 67.7 & 58.8 & 45.7 & 57.3 & 62.7 & 48.6\\

\midrule \multirow{3}{*}{Tent}
&\xmark   &52.4 & 51.9 & 53.3 & 53.8 & 48.1 & 57.0 & 46.2 & 10.3 & 53.5 & 67.9 & 74.2 & 67.1 & 52.3 & 66.5 & 64.9 & 54.6\\
&\cmark  &55.0 & 55.0 & 56.2 & 57.1 & 54.6 & 61.6 & 49.3 & 62.9 & 64.0 & 72.3 & 78.1 & 69.3 & 62.7 & 71.3 & 69.0 & 62.6\\
&Improve  &\textcolor{mycolor4}{$\bigtriangleup2.5$} & \textcolor{mycolor4}{$\bigtriangleup3.2$} & \textcolor{mycolor4}{$\bigtriangleup2.9$} & \textcolor{mycolor4}{$\bigtriangleup3.4$} & \textcolor{mycolor4}{$\bigtriangleup6.5$} & \textcolor{mycolor4}{$\bigtriangleup4.6$} & \textcolor{mycolor4}{$\bigtriangleup3.1$} & \textcolor{mycolor4}{$\bigtriangleup52.6$} & \textcolor{mycolor4}{$\bigtriangleup10.6$} & \textcolor{mycolor4}{$\bigtriangleup4.4$} & \textcolor{mycolor4}{$\bigtriangleup4.0$} & \textcolor{mycolor4}{$\bigtriangleup2.2$} & \textcolor{mycolor4}{$\bigtriangleup10.4$} & \textcolor{mycolor4}{$\bigtriangleup4.9$} & \textcolor{mycolor4}{$\bigtriangleup4.1$} & \textcolor{mycolor4}{$\bigtriangleup7.9$}\\
\midrule \multirow{3}{*}{SAR}
&\xmark   &51.8 & 51.7 & 52.7 & 51.9 & 48.2 & 55.6 & 47.8 & 20.3 & 52.9 & 66.8 & 73.2 & 66.0 & 52.2 & 64.1 & 62.8 & 54.5\\
&\cmark  &56.0 & 56.0 & 57.2 & 58.0 & 56.3 & 62.3 & 54.1 & 64.0 & 64.3 & 72.4 & 78.3 & 69.6 & 64.0 & 71.5 & 69.1 & 63.5\\
&Improve  &\textcolor{mycolor4}{$\bigtriangleup4.2$} & \textcolor{mycolor4}{$\bigtriangleup4.4$} & \textcolor{mycolor4}{$\bigtriangleup4.5$} & \textcolor{mycolor4}{$\bigtriangleup6.2$} & \textcolor{mycolor4}{$\bigtriangleup8.1$} & \textcolor{mycolor4}{$\bigtriangleup6.6$} & \textcolor{mycolor4}{$\bigtriangleup6.3$} & \textcolor{mycolor4}{$\bigtriangleup43.8$} & \textcolor{mycolor4}{$\bigtriangleup11.4$} & \textcolor{mycolor4}{$\bigtriangleup5.7$} & \textcolor{mycolor4}{$\bigtriangleup5.1$} & \textcolor{mycolor4}{$\bigtriangleup3.6$} & \textcolor{mycolor4}{$\bigtriangleup11.8$} & \textcolor{mycolor4}{$\bigtriangleup7.4$} & \textcolor{mycolor4}{$\bigtriangleup6.2$} & \textcolor{mycolor4}{$\bigtriangleup9.0$}\\
\midrule \multirow{3}{*}{EATA}
&\xmark   &52.0 & 53.6 & 53.9 & 49.3 & 49.5 & 54.4 & 55.6 & 58.1 & 56.9 & 69.6 & 74.9 & 63.6 & 61.1 & 68.0 & 64.2 & 59.0\\
&\cmark  &54.9 & 56.4 & 54.7 & 56.5 & 56.3 & 62.1 & 59.0 & 67.0 & 65.4 & 73.4 & 78.4 & 68.0 & 68.0 & 73.0 & 70.4 & 64.2\\
&Improve  &\textcolor{mycolor4}{$\bigtriangleup2.8$} & \textcolor{mycolor4}{$\bigtriangleup2.8$} & \textcolor{mycolor4}{$\bigtriangleup0.8$} & \textcolor{mycolor4}{$\bigtriangleup7.3$} & \textcolor{mycolor4}{$\bigtriangleup6.8$} & \textcolor{mycolor4}{$\bigtriangleup7.7$} & \textcolor{mycolor4}{$\bigtriangleup3.4$} & \textcolor{mycolor4}{$\bigtriangleup8.9$} & \textcolor{mycolor4}{$\bigtriangleup8.5$} & \textcolor{mycolor4}{$\bigtriangleup3.9$} & \textcolor{mycolor4}{$\bigtriangleup3.5$} & \textcolor{mycolor4}{$\bigtriangleup4.4$} & \textcolor{mycolor4}{$\bigtriangleup6.8$} & \textcolor{mycolor4}{$\bigtriangleup5.0$} & \textcolor{mycolor4}{$\bigtriangleup6.1$} & \textcolor{mycolor4}{$\bigtriangleup5.3$}\\
\midrule \multirow{3}{*}{COTTA}
&\xmark   &42.9 & 40.0 & 44.6 & 36.0 & 29.7 & 44.8 & 37.2 & 42.3 & 46.4 & 60.7 & 72.9 & 65.0 & 45.4 & 61.6 & 62.9 & 48.8\\
&\cmark  &51.6 & 49.0 & 52.9 & 41.7 & 37.0 & 51.6 & 43.8 & 46.7 & 53.2 & 65.9 & 74.4 & 65.6 & 52.8 & 66.7 & 65.9 & 54.6\\
&Improve  &\textcolor{mycolor4}{$\bigtriangleup8.6$} & \textcolor{mycolor4}{$\bigtriangleup9.0$} & \textcolor{mycolor4}{$\bigtriangleup8.3$} & \textcolor{mycolor4}{$\bigtriangleup5.7$} & \textcolor{mycolor4}{$\bigtriangleup7.4$} & \textcolor{mycolor4}{$\bigtriangleup6.8$} & \textcolor{mycolor4}{$\bigtriangleup6.6$} & \textcolor{mycolor4}{$\bigtriangleup4.4$} & \textcolor{mycolor4}{$\bigtriangleup6.9$} & \textcolor{mycolor4}{$\bigtriangleup5.2$} & \textcolor{mycolor4}{$\bigtriangleup1.5$} & \textcolor{mycolor4}{$\bigtriangleup0.5$} & \textcolor{mycolor4}{$\bigtriangleup7.4$} & \textcolor{mycolor4}{$\bigtriangleup5.1$} & \textcolor{mycolor4}{$\bigtriangleup3.0$} & \textcolor{mycolor4}{$\bigtriangleup5.8$}\\
\midrule \multirow{3}{*}{MEMO}
&\xmark   &39.7 & 36.5 & 39.8 & 32.4 & 25.8 & 40.3 & 34.7 & 27.5 & 32.8 & 53.5 & 66.2 & 56.0 & 35.7 & 55.9 & 58.2 & 42.3\\
&\cmark  &40.6 & 37.5 & 40.6 & 33.4 & 26.7 & 41.2 & 35.4 & 28.7 & 33.7 & 54.7 & 67.1 & 55.9 & 36.6 & 57.2 & 59.2 & 43.2\\
&Improve  &\textcolor{mycolor4}{$\bigtriangleup0.8$} & \textcolor{mycolor4}{$\bigtriangleup1.0$} & \textcolor{mycolor4}{$\bigtriangleup0.8$} & \textcolor{mycolor4}{$\bigtriangleup1.0$} & \textcolor{mycolor4}{$\bigtriangleup0.9$} & \textcolor{mycolor4}{$\bigtriangleup1.0$} & \textcolor{mycolor4}{$\bigtriangleup0.7$} & \textcolor{mycolor4}{$\bigtriangleup1.2$} & \textcolor{mycolor4}{$\bigtriangleup0.9$} & \textcolor{mycolor4}{$\bigtriangleup1.2$} & \textcolor{mycolor4}{$\bigtriangleup0.8$} & -& \textcolor{mycolor4}{$\bigtriangleup0.9$} & \textcolor{mycolor4}{$\bigtriangleup1.3$} & \textcolor{mycolor4}{$\bigtriangleup1.1$} & \textcolor{mycolor4}{$\bigtriangleup0.9$}\\
\bottomrule
\end{tabular}
}\vspace{-5pt}
\end{table}

\begin{table}[htbp]
    \centering
    \vspace{-5pt}
    \caption{Comparison w.r.t. mixed shifts performance. Results obtained on ViT and ImageNet-C (level 5) under \textbf{mixed shifts} TTA setting, the performance is evaluated on a single data stream consisting of 15 mixed corruptions. Substantial ($\geq 0.5$) \textcolor{mycolor4}{improvement} and \textcolor{mycolor3}{degradation} compared to the baseline are highlighted in blue or red respectively.}
    \label{tab:mixed domain shifts}

    \resizebox{\textwidth}{!}{
    \setlength{\tabcolsep}{2.5mm}
    \begin{tabular}{c|cc|cc|cc|cc|cc}
    \toprule
       & \multicolumn{2}{c}{TENT} & \multicolumn{2}{c}{SAR} & \multicolumn{2}{c}{EATA} & \multicolumn{2}{c}{CoTTA}&  \multicolumn{2}{c}{Avg.} \\ \midrule\texttt{COME} & Acc$\uparrow$ &FPR$\downarrow$ & Acc$\uparrow$ &FPR$\downarrow$ & Acc$\uparrow$ &FPR$\downarrow$ & Acc$\uparrow$ &FPR$\downarrow$ & Acc$\uparrow$ &FPR$\downarrow$  \\

\midrule

\xmark   &58.0&72.3&53.6&68.2&58.8&71.3&62.0&69.7&58.1&70.4\\
\cmark  &61.2&67.9&62.3&66.9&61.8&67.0&65.1&70.7&62.6&68.1\\
Improve  &\textcolor{mycolor4}{$\bigtriangleup3.2$} &\textcolor{mycolor4}{$\bigtriangledown4.4$} &\textcolor{mycolor4}{$\bigtriangleup8.6$} &\textcolor{mycolor4}{$\bigtriangledown1.3$} &\textcolor{mycolor4}{$\bigtriangleup3.0$} &\textcolor{mycolor4}{$\bigtriangledown4.3$} &\textcolor{mycolor4}{$\bigtriangleup3.1$} &\textcolor{mycolor1}{$\bigtriangleup0.9$} &\textcolor{mycolor4}{$\bigtriangleup4.5$} &\textcolor{mycolor4}{$\bigtriangledown2.3$}\\

    \bottomrule
    \end{tabular}
    }\vspace{-5pt}
    \end{table}

\vspace{7em}

\subsection{Additional results w.r.t FPR (supplementary to Table  \ref{tab:standard})  }
\label{fprs}
We supplement the uncertainty estimation performance from Table \ref{tab:standard}, as shown in Tables \ref{tab:fprs}. It is evident that our \texttt{COME} method excels in the classification task while simultaneously improving the average uncertainty estimation performance of the  EM-based methods.

\begin{table}[htbp]
\vspace{-5pt}
\centering
\caption{Uncertainty estimation performance comparison \textbf{(FPR)} on ImageNet-C(severity level5) as the full result of Table \ref{tab:standard}. Substantial ($\geq 0.5$) \textcolor{mycolor4}{improvement} and \textcolor{mycolor3}{degradation} compared to the baseline are highlighted in blue or red respectively.}
\label{tab:fprs}
\resizebox{\textwidth}{!}{
\setlength{\tabcolsep}{1mm}
\begin{tabular}{cc|ccc|cccc|cccc|cccc|c}
\toprule
 & & \multicolumn{3}{c}{Noise} & \multicolumn{4}{c}{Blur} & \multicolumn{4}{c}{Weather} & \multicolumn{4}{c}{Digital} & \multicolumn{1}{c}{}  \\ 
\cmidrule(r){3-5} \cmidrule(r){6-9} \cmidrule(r){10-13} \cmidrule(r){14-17} \cmidrule(r){18-18}  
 Methods & UEM & Gauss. & Shot & Impul. & Defoc & Glass & Motion & Zoom & Snow & Frost &Fog & Brit. &Contr. &Elast. &Pixel & JPEG & Avg. \\

\midrule No Adapt

&\xmark   &64.6 & 65.8 & 67.1 & 68.6 & 69.7 & 66.6 & 73.0 & 69.1 & 64.1 & 66.4 & 63.2 & 73.4 & 70.9 & 64.0 & 66.1 & 67.5\\
\midrule
PL
&\xmark   &66.5 & 65.9 & 66.6 & 67.3 & 70.4 & 67.6 & 71.5 & 88.4 & 78.2 & 64.9 & 63.1 & 63.9 & 70.3 & 64.7 & 67.3 & 69.1\\

T3A
&\xmark   &67.2 & 68.2 & 68.5 & 64.5 & 69.8 & 66.3 & 70.3 & 66.2 & 66.7 & 63.3 & 68.7 & 68.3 & 67.8 & 69.6 & 69.9 & 67.7\\

TEA
&\xmark   &67.2 & 68.3 & 67.2 & 70.0 & 69.3 & 67.1 & 68.5 & 94.0 & 65.2 & 63.9 & 62.8 & 65.2 & 65.6 & 63.7 & 65.9 & 68.3\\

LAME
&\xmark   &65.0 & 66.5 & 67.6 & 70.3 & 71.9 & 67.9 & 75.8 & 73.1 & 66.9 & 70.3 & 63.6 & 80.7 & 74.1 & 64.6 & 67.2 & 69.7\\

FOA
&\xmark   &65.6 & 65.0 & 65.4 & 66.3 & 65.4 & 63.9 & 68.1 & 60.2 & 62.1 & 60.2 & 61.3 & 60.6 & 63.7 & 62.1 & 63.6 & 63.6\\

\midrule \multirow{3}{*}{Tent}
&\xmark   &67.6 & 67.2 & 67.2 & 67.1 & 68.7 & 67.7 & 72.2 & 95.1 & 87.7 & 62.7 & 62.3 & 65.2 & 68.7 & 65.3 & 67.0 & 70.1\\
&\cmark  &65.9 & 66.5 & 67.2 & 68.1 & 69.1 & 67.5 & 70.8 & 70.6 & 65.1 & 61.7 & 63.1 & 63.9 & 67.5 & 64.4 & 65.9 & 66.5\\
&Improve  &\textcolor{mycolor4}{$\bigtriangledown1.7$} & \textcolor{mycolor4}{$\bigtriangledown0.7$} & - & \textcolor{mycolor1}{$\bigtriangleup1.1$} & - &  -  & \textcolor{mycolor4}{$\bigtriangledown1.4$} & \textcolor{mycolor4}{$\bigtriangledown24.5$} & \textcolor{mycolor4}{$\bigtriangledown22.7$} & \textcolor{mycolor4}{$\bigtriangledown1.0$} & \textcolor{mycolor1}{$\bigtriangleup0.8$} & \textcolor{mycolor4}{$\bigtriangledown1.3$} & \textcolor{mycolor4}{$\bigtriangledown1.3$} & \textcolor{mycolor4}{$\bigtriangledown0.9$} & \textcolor{mycolor4}{$\bigtriangledown1.1$} & \textcolor{mycolor4}{$\bigtriangledown3.6$}\\
\midrule \multirow{3}{*}{EATA}
&\xmark   &66.7 & 67.4 & 67.1 & 67.2 & 67.4 & 66.6 & 65.4 & 62.8 & 62.8 & 63.4 & 63.3 & 63.1 & 64.2 & 64.2 & 64.2 & 65.1\\
&\cmark  &63.9 & 63.7 & 63.6 & 66.6 & 67.1 & 64.8 & 66.9 & 62.5 & 61.9 & 61.2 & 61.3 & 62.6 & 64.3 & 62.0 & 64.4 & 63.8\\
&Improve  &\textcolor{mycolor4}{$\bigtriangledown2.8$} & \textcolor{mycolor4}{$\bigtriangledown3.7$} & \textcolor{mycolor4}{$\bigtriangledown3.5$} &  -  &  -  & \textcolor{mycolor4}{$\bigtriangledown1.9$} & \textcolor{mycolor1}{$\bigtriangleup1.4$} &  -  & \textcolor{mycolor4}{$\bigtriangledown1.0$} & \textcolor{mycolor4}{$\bigtriangledown2.2$} & \textcolor{mycolor4}{$\bigtriangledown2.0$} & \textcolor{mycolor4}{$\bigtriangledown0.6$} &  -  & \textcolor{mycolor4}{$\bigtriangledown2.2$} &  -  & \textcolor{mycolor4}{$\bigtriangledown1.3$}\\
\midrule \multirow{3}{*}{SAR}
&\xmark   &64.9 & 64.2 & 65.4 & 65.6 & 66.3 & 65.4 & 67.6 & 86.9 & 71.8 & 61.5 & 62.9 & 63.6 & 65.2 & 63.5 & 65.1 & 66.7\\
&\cmark  &63.1 & 63.0 & 63.9 & 66.0 & 65.8 & 65.2 & 67.4 & 62.2 & 62.1 & 60.5 & 61.2 & 63.6 & 65.8 & 63.4 & 63.9 & 63.8\\
&Improve  &\textcolor{mycolor4}{$\bigtriangledown1.9$} & \textcolor{mycolor4}{$\bigtriangledown1.2$} & \textcolor{mycolor4}{$\bigtriangledown1.5$} &  -  &  -  &  -  &  -  & \textcolor{mycolor4}{$\bigtriangledown24.7$} & \textcolor{mycolor4}{$\bigtriangledown9.7$} & \textcolor{mycolor4}{$\bigtriangledown1.0$} & \textcolor{mycolor4}{$\bigtriangledown1.7$} &  - & \textcolor{mycolor1}{$\bigtriangleup0.6$} &  -  & \textcolor{mycolor4}{$\bigtriangledown1.2$} & \textcolor{mycolor4}{$\bigtriangledown2.9$}\\
\midrule \multirow{3}{*}{COTTA}
&\xmark   &65.6 & 66.0 & 67.8 & 70.7 & 71.5 & 67.3 & 75.3 & 66.5 & 64.5 & 65.1 & 62.9 & 70.8 & 72.3 & 64.8 & 67.1 & 67.9\\
&\cmark  &64.4 & 65.0 & 64.6 & 69.9 & 73.1 & 66.8 & 76.3 & 69.9 & 68.2 & 64.9 & 62.6 & 64.4 & 74.0 & 63.3 & 65.5 & 67.5\\
&Improve  &\textcolor{mycolor4}{$\bigtriangledown1.2$} & \textcolor{mycolor4}{$\bigtriangledown1.0$} & \textcolor{mycolor4}{$\bigtriangledown3.2$} & \textcolor{mycolor4}{$\bigtriangledown0.8$} & \textcolor{mycolor1}{$\bigtriangleup1.5$} &  -  & \textcolor{mycolor1}{$\bigtriangleup1.0$} & \textcolor{mycolor1}{$\bigtriangleup3.3$} & \textcolor{mycolor1}{$\bigtriangleup3.7$} &  - &  -  & \textcolor{mycolor4}{$\bigtriangledown6.4$} & \textcolor{mycolor1}{$\bigtriangleup1.7$} & \textcolor{mycolor4}{$\bigtriangledown1.5$} & \textcolor{mycolor4}{$\bigtriangledown1.5$} &  - \\
\midrule \multirow{3}{*}{MEMO}
&\xmark   &65.6 & 65.1 & 67.7 & 70.3 & 76.7 & 71.3 & 78.7 & 77.1 & 73.9 & 72.7 & 70.2 & 71.0 & 78.5 & 72.0 & 71.2 & 72.1\\
&\cmark  &63.0 & 63.5 & 65.5 & 69.8 & 73.4 & 69.5 & 77.9 & 74.8 & 72.8 & 71.7 & 69.3 & 70.4 & 77.7 & 71.8 & 70.6 & 70.8\\
&Improve  &\textcolor{mycolor4}{$\bigtriangledown2.7$} & \textcolor{mycolor4}{$\bigtriangledown1.5$} & \textcolor{mycolor4}{$\bigtriangledown2.2$} &  -  & \textcolor{mycolor4}{$\bigtriangledown3.3$} & \textcolor{mycolor4}{$\bigtriangledown1.8$} & \textcolor{mycolor4}{$\bigtriangledown0.9$} & \textcolor{mycolor4}{$\bigtriangledown2.3$} & \textcolor{mycolor4}{$\bigtriangledown1.1$} & \textcolor{mycolor4}{$\bigtriangledown0.9$} & \textcolor{mycolor4}{$\bigtriangledown0.9$} & \textcolor{mycolor4}{$\bigtriangledown0.6$} & \textcolor{mycolor4}{$\bigtriangledown0.7$} &  -  & \textcolor{mycolor4}{$\bigtriangledown0.6$} & \textcolor{mycolor4}{$\bigtriangledown1.3$}\\

\bottomrule
\end{tabular}
}\vspace{-5pt}
\end{table}

\subsection{Additional results w.r.t FPR (supplementary to Table \ref{tab:lifelong}) }
\label{lifelong-fprs}
We supplement the uncertainty estimation performance from Table \ref{tab:lifelong}, as shown in Table \ref{tab:life-long fprs}. According to the experimental results, our \texttt{COME} method excels in the classification task while consistently improving the uncertainty estimation performance of existing EM-based methods.
\begin{table}[htbp]
\vspace{-5pt}
\centering
\caption{Uncertainty estimation performance \textbf{(FPR)} comparison under \textbf{lifelong} TTA setting  as the full result of Table \ref{tab:lifelong}. Substantial ($\geq 0.5$) \textcolor{mycolor4}{improvement} and \textcolor{mycolor3}{degradation} compared to the baseline are highlighted in blue or red respectively.}
\label{tab:life-long fprs}
\resizebox{\textwidth}{!}{
\setlength{\tabcolsep}{1mm}
\begin{tabular}{cc|ccc|cccc|cccc|cccc|c}
\toprule
 & & \multicolumn{3}{c}{Noise} & \multicolumn{4}{c}{Blur} & \multicolumn{4}{c}{Weather} & \multicolumn{4}{c}{Digital} & \multicolumn{1}{c}{}  \\ 
\cmidrule(r){3-5} \cmidrule(r){6-9} \cmidrule(r){10-13} \cmidrule(r){14-17} \cmidrule(r){18-18}  
 Methods & UEM & Gauss. & Shot & Impul. & Defoc & Glass & Motion & Zoom & Snow & Frost &Fog & Brit. &Contr. &Elast. &Pixel & JPEG & Avg. \\

\midrule No Adapt
&\xmark   &64.6 & 65.2 & 65.8 & 66.6 & 67.9 & 67.6 & 68.7 & 69.4 & 69.2 & 67.9 & 66.1 & 72.1 & 72.5 & 71.8 & 71.1 & 68.4\\
\midrule
PL
&\xmark   &66.5 & 67.3 & 67.8 & 68.6 & 69.9 & 69.8 & 70.1 & 70.8 & 70.4 & 69.7 & 68.7 & 68.6 & 69.2 & 69.0 & 68.9 & 69.0\\

FOA
&\xmark   &65.0 & 65.5 & 65.3 & 66.0 & 66.6 & 66.3 & 66.7 & 65.9 & 65.6 & 64.7 & 63.3 & 62.9 & 63.0 & 62.9 & 62.8 & 64.8\\

\midrule \multirow{3}{*}{Tent}
&\xmark   &67.3 & 68.2 & 69.0 & 72.0 & 74.8 & 72.8 & 74.0 & 79.2 & 72.6 & 70.2 & 67.6 & 71.2 & 75.3 & 70.5 & 71.9 & 71.8\\
&\cmark  &66.9 & 66.7 & 67.0 & 67.3 & 67.8 & 67.9 & 68.1 & 67.9 & 67.5 & 67.2 & 66.2 & 66.7 & 66.6 & 66.4 & 66.3 & 67.1\\
&Improve  & - & \textcolor{mycolor4}{$\bigtriangledown1.5$} & \textcolor{mycolor4}{$\bigtriangledown2.0$} & \textcolor{mycolor4}{$\bigtriangledown4.7$} & \textcolor{mycolor4}{$\bigtriangledown6.9$} & \textcolor{mycolor4}{$\bigtriangledown5.0$} & \textcolor{mycolor4}{$\bigtriangledown5.8$} & \textcolor{mycolor4}{$\bigtriangledown11.3$} & \textcolor{mycolor4}{$\bigtriangledown5.1$} & \textcolor{mycolor4}{$\bigtriangledown3.0$} & \textcolor{mycolor4}{$\bigtriangledown1.4$} & \textcolor{mycolor4}{$\bigtriangledown4.6$} & \textcolor{mycolor4}{$\bigtriangledown8.6$} & \textcolor{mycolor4}{$\bigtriangledown4.1$} & \textcolor{mycolor4}{$\bigtriangledown5.5$} & \textcolor{mycolor4}{$\bigtriangledown4.7$}\\

\midrule \multirow{3}{*}{EATA}
&\xmark   &68.6 & 70.8 & 70.0 & 72.8 & 73.2 & 72.1 & 72.3 & 70.3 & 70.3 & 69.1 & 65.9 & 70.1 & 70.0 & 69.6 & 69.2 & 70.3\\
&\cmark  &66.8 & 67.0 & 66.9 & 67.2 & 67.4 & 67.5 & 67.6 & 67.2 & 66.9 & 66.5 & 65.8 & 66.0 & 66.0 & 65.8 & 65.8 & 66.7\\
&Improve  &\textcolor{mycolor4}{$\bigtriangledown1.8$} & \textcolor{mycolor4}{$\bigtriangledown3.8$} & \textcolor{mycolor4}{$\bigtriangledown3.1$} & \textcolor{mycolor4}{$\bigtriangledown5.6$} & \textcolor{mycolor4}{$\bigtriangledown5.9$} & \textcolor{mycolor4}{$\bigtriangledown4.6$} & \textcolor{mycolor4}{$\bigtriangledown4.7$} & \textcolor{mycolor4}{$\bigtriangledown3.0$} & \textcolor{mycolor4}{$\bigtriangledown3.4$} & \textcolor{mycolor4}{$\bigtriangledown2.6$} &  - & \textcolor{mycolor4}{$\bigtriangledown4.2$} & \textcolor{mycolor4}{$\bigtriangledown4.0$} & \textcolor{mycolor4}{$\bigtriangledown3.8$} & \textcolor{mycolor4}{$\bigtriangledown3.4$} & \textcolor{mycolor4}{$\bigtriangledown3.6$}\\
\midrule \multirow{3}{*}{SAR}
&\xmark   &65.5 & 65.7 & 65.5 & 67.1 & 68.0 & 67.2 & 69.3 & 76.1 & 92.8 & 100.0 & 99.6 & 75.7 & 100.0 & 100.0 & 99.9 & 80.8\\
&\cmark  &65.0 & 64.5 & 65.0 & 65.8 & 66.1 & 66.3 & 66.7 & 66.4 & 66.1 & 65.9 & 65.1 & 65.8 & 65.7 & 65.5 & 65.3 & 65.7\\
&Improve  & -  & \textcolor{mycolor4}{$\bigtriangledown1.2$} & -  & \textcolor{mycolor4}{$\bigtriangledown1.2$} & \textcolor{mycolor4}{$\bigtriangledown1.8$} & \textcolor{mycolor4}{$\bigtriangledown0.9$} & \textcolor{mycolor4}{$\bigtriangledown2.6$} & \textcolor{mycolor4}{$\bigtriangledown9.8$} & \textcolor{mycolor4}{$\bigtriangledown26.7$} & \textcolor{mycolor4}{$\bigtriangledown34.0$} & \textcolor{mycolor4}{$\bigtriangledown34.4$} & \textcolor{mycolor4}{$\bigtriangledown9.9$} & \textcolor{mycolor4}{$\bigtriangledown34.3$} & \textcolor{mycolor4}{$\bigtriangledown34.5$} & \textcolor{mycolor4}{$\bigtriangledown34.6$} & \textcolor{mycolor4}{$\bigtriangledown15.1$}\\
\midrule \multirow{3}{*}{COTTA}
&\xmark   &65.2 & 66.7 & 67.7 & 70.0 & 72.4 & 71.5 & 74.9 & 70.3 & 69.0 & 69.1 & 67.6 & 69.9 & 73.8 & 69.9 & 70.1 & 69.9\\
&\cmark  &64.4 & 64.4 & 64.6 & 65.0 & 65.4 & 65.5 & 66.6 & 66.5 & 66.2 & 66.1 & 65.7 & 66.0 & 66.3 & 66.2 & 66.3 & 65.7\\
&Improve  &\textcolor{mycolor4}{$\bigtriangledown0.9$} & \textcolor{mycolor4}{$\bigtriangledown2.3$} & \textcolor{mycolor4}{$\bigtriangledown3.1$} & \textcolor{mycolor4}{$\bigtriangledown5.1$} & \textcolor{mycolor4}{$\bigtriangledown7.0$} & \textcolor{mycolor4}{$\bigtriangledown6.0$} & \textcolor{mycolor4}{$\bigtriangledown8.4$} & \textcolor{mycolor4}{$\bigtriangledown3.8$} & \textcolor{mycolor4}{$\bigtriangledown2.8$} & \textcolor{mycolor4}{$\bigtriangledown3.0$} & \textcolor{mycolor4}{$\bigtriangledown1.9$} & \textcolor{mycolor4}{$\bigtriangledown3.9$} & \textcolor{mycolor4}{$\bigtriangledown7.5$} & \textcolor{mycolor4}{$\bigtriangledown3.7$} & \textcolor{mycolor4}{$\bigtriangledown3.8$} & \textcolor{mycolor4}{$\bigtriangledown4.2$}\\

\bottomrule
\end{tabular}
}\vspace{-5pt}
\end{table}

\vspace{5em}

\subsection{Additional results w.r.t. open-world TTA setting.(supplementary to Table \ref{tab:openworld})}
\label{open-world-0.3}
We add comparisons of classification and uncertainty estimation under \textbf{open-world} TTA settings with different mix ratios. As observed under $P^{\rm test}=0.7 P^{\rm Cov}+0.3 P^{\rm Sem}$ setting, our \texttt{COME} method consistently enhances the performance of both classification and uncertainty estimation tasks when applied to existing EM-based methods.
\begin{table}[!htbp]
\vspace{-5pt}
    \centering
    \caption{Classification and uncertainty estimation comparisons under \textbf{open-world} TTA settings, , where $P^{\rm test}=0.7 P^{\rm Cov}+0.3 P^{\rm Sem}$(Gaussian noise of severity level 3) and a suit of diverse abnormal outliers as same with Table \ref{tab:openworld}.}
    \label{tab:add-openworld}
    \resizebox{\textwidth}{!}{
    \setlength{\tabcolsep}{1.75mm}
    \begin{tabular}{cc|cc|cc|cc|cc|cc|cc|cc}
    \toprule
      & & \multicolumn{2}{c}{None} & \multicolumn{2}{c}{NINCO} & \multicolumn{2}{c}{iNaturist} & \multicolumn{2}{c}{SSB-Hard} & \multicolumn{2}{c}{Texture} & \multicolumn{2}{c}{Places}  & \multicolumn{2}{c}{Avg.} \\ \midrule Method&COME & Acc$\uparrow$ &FPR$\downarrow$ & Acc$\uparrow$ &FPR$\downarrow$ & Acc$\uparrow$ &FPR$\downarrow$ & Acc$\uparrow$ &FPR$\downarrow$ & Acc$\uparrow$ &FPR$\downarrow$ & Acc$\uparrow$ &FPR$\downarrow$ & Acc$\uparrow$ &FPR$\downarrow$ \\
      \midrule 
No Adapt
& \xmark &64.4&63.7&64.2&68.2&64.4&67.9&64.4&70.2&63.9&63.6&64.4&59.5&64.3&65.5\\
\midrule 
PL
& \xmark &69.0&63.3&66.5&68.0&69.2&66.9&68.7&71.9&65.7&64.0&68.8&58.0&68.0&65.4\\

T3A
& \xmark &64.4&71.4&64.1&76.1&64.2&73.4&64.2&77.9&63.7&72.5&64.2&71.2&64.1&73.7\\

TEA
& \xmark &64.0&63.5&61.4&69.6&63.8&71.2&63.0&75.7&61.4&66.4&64.0&63.8&62.9&68.4\\

LAME
& \xmark &64.1&63.8&63.9&68.9&64.1&69.6&64.1&70.7&63.6&65.2&64.2&61.8&64.0&66.7\\

FOA
& \xmark &67.8&62.4&66.8&66.7&67.6&65.1&67.7&71.4&66.5&63.1&67.7&56.1&67.3&64.1\\

\midrule \multirow{3}{*}{Tent}
&\xmark   &70.7&63.1&68.0&67.4&70.8&67.1&70.4&72.8&67.6&64.8&70.3&59.2&69.6&65.7\\
&\cmark  &72.6&64.9&70.2&62.3&72.6&63.2&72.7&67.8&69.6&59.1&72.0&48.4&71.6&61.0\\
&Improve  &\textcolor{mycolor4}{$\bigtriangleup1.9$} &\textcolor{mycolor1}{$\bigtriangleup1.8$} &\textcolor{mycolor4}{$\bigtriangleup2.2$} &\textcolor{mycolor4}{$\bigtriangledown5.1$} &\textcolor{mycolor4}{$\bigtriangleup1.8$} &\textcolor{mycolor4}{$\bigtriangledown3.9$} &\textcolor{mycolor4}{$\bigtriangleup2.3$} &\textcolor{mycolor4}{$\bigtriangledown5.0$} &\textcolor{mycolor4}{$\bigtriangleup2.1$} &\textcolor{mycolor4}{$\bigtriangledown5.7$} &\textcolor{mycolor4}{$\bigtriangleup1.7$} &\textcolor{mycolor4}{$\bigtriangledown10.7$} &\textcolor{mycolor4}{$\bigtriangleup2.0$} &\textcolor{mycolor4}{$\bigtriangledown4.8$}\\

\midrule \multirow{3}{*}{EATA}
&\xmark   &70.4&63.6&67.9&67.6&70.4&69.2&70.2&73.6&67.6&64.4&69.9&62.2&69.4&66.8\\
&\cmark  &73.3&63.3&71.4&60.2&73.3&63.2&73.1&68.0&71.9&58.2&73.0&49.4&72.7&60.4\\
&Improve  &\textcolor{mycolor4}{$\bigtriangleup2.9$} & - &\textcolor{mycolor4}{$\bigtriangleup3.4$} &\textcolor{mycolor4}{$\bigtriangledown7.3$} &\textcolor{mycolor4}{$\bigtriangleup2.9$} &\textcolor{mycolor4}{$\bigtriangledown6.0$} &\textcolor{mycolor4}{$\bigtriangleup2.9$} &\textcolor{mycolor4}{$\bigtriangledown5.6$} &\textcolor{mycolor4}{$\bigtriangleup4.3$} &\textcolor{mycolor4}{$\bigtriangledown6.2$} &\textcolor{mycolor4}{$\bigtriangleup3.1$} &\textcolor{mycolor4}{$\bigtriangledown12.8$} &\textcolor{mycolor4}{$\bigtriangleup3.3$} &\textcolor{mycolor4}{$\bigtriangledown6.4$}\\

\midrule \multirow{3}{*}{SAR}
&\xmark   &69.0&62.6&66.3&67.6&68.0&67.8&68.0&73.5&66.7&62.3&67.9&59.9&67.6&65.6\\
&\cmark  &73.1&62.1&70.8&62.6&73.3&65.2&73.5&68.4&70.6&59.0&73.1&52.9&72.4&61.7\\
&Improve  &\textcolor{mycolor4}{$\bigtriangleup4.1$} &\textcolor{mycolor4}{$\bigtriangledown0.6$} &\textcolor{mycolor4}{$\bigtriangleup4.5$} &\textcolor{mycolor4}{$\bigtriangledown5.0$} &\textcolor{mycolor4}{$\bigtriangleup5.3$} &\textcolor{mycolor4}{$\bigtriangledown2.6$} &\textcolor{mycolor4}{$\bigtriangleup5.6$} &\textcolor{mycolor4}{$\bigtriangledown5.1$} &\textcolor{mycolor4}{$\bigtriangleup3.8$} &\textcolor{mycolor4}{$\bigtriangledown3.3$} &\textcolor{mycolor4}{$\bigtriangleup5.3$} &\textcolor{mycolor4}{$\bigtriangledown6.9$} &\textcolor{mycolor4}{$\bigtriangleup4.8$} &\textcolor{mycolor4}{$\bigtriangledown3.9$}\\

\midrule \multirow{3}{*}{COTTA}
&\xmark   &67.6&64.1&65.8&68.3&69.0&67.6&69.0&71.9&65.3&64.3&68.7&59.9&67.6&66.0\\
&\cmark  &70.5&63.0&67.3&65.7&71.5&67.8&71.7&72.3&66.9&62.3&70.9&55.6&69.8&64.5\\
&Improve  &\textcolor{mycolor4}{$\bigtriangleup2.9$} &\textcolor{mycolor4}{$\bigtriangledown1.0$} &\textcolor{mycolor4}{$\bigtriangleup1.5$} &\textcolor{mycolor4}{$\bigtriangledown2.6$} &\textcolor{mycolor4}{$\bigtriangleup2.5$} & - &\textcolor{mycolor4}{$\bigtriangleup2.6$} & -  &\textcolor{mycolor4}{$\bigtriangleup1.6$} &\textcolor{mycolor4}{$\bigtriangledown1.9$} &\textcolor{mycolor4}{$\bigtriangleup2.1$} &\textcolor{mycolor4}{$\bigtriangledown4.3$} &\textcolor{mycolor4}{$\bigtriangleup2.2$} &\textcolor{mycolor4}{$\bigtriangledown1.5$}\\

    \bottomrule
    \end{tabular}
    }\vspace{-5pt}
    \end{table}

\subsection{Additional performance visualization of comparison results.(supplementary to Figure \ref{fig:motivation2})}
\label{visualization}
We provide more results regarding comparison on two representative TTA methods, i.e., the seminal Tent~\citep{wang2020tent} and recent SOTA SAR~\citep{niu2023towards}. Our \texttt{COME} method consistently improves the performance of both classification and uncertainty estimation tasks across various corruptions, where the proposed \texttt{COME} explicitly maintains conservative predictive confidence during TTA.
\begin{figure}[htbp]
    \vspace{-2pt}
    \centering
    \includegraphics[width=0.99\textwidth]{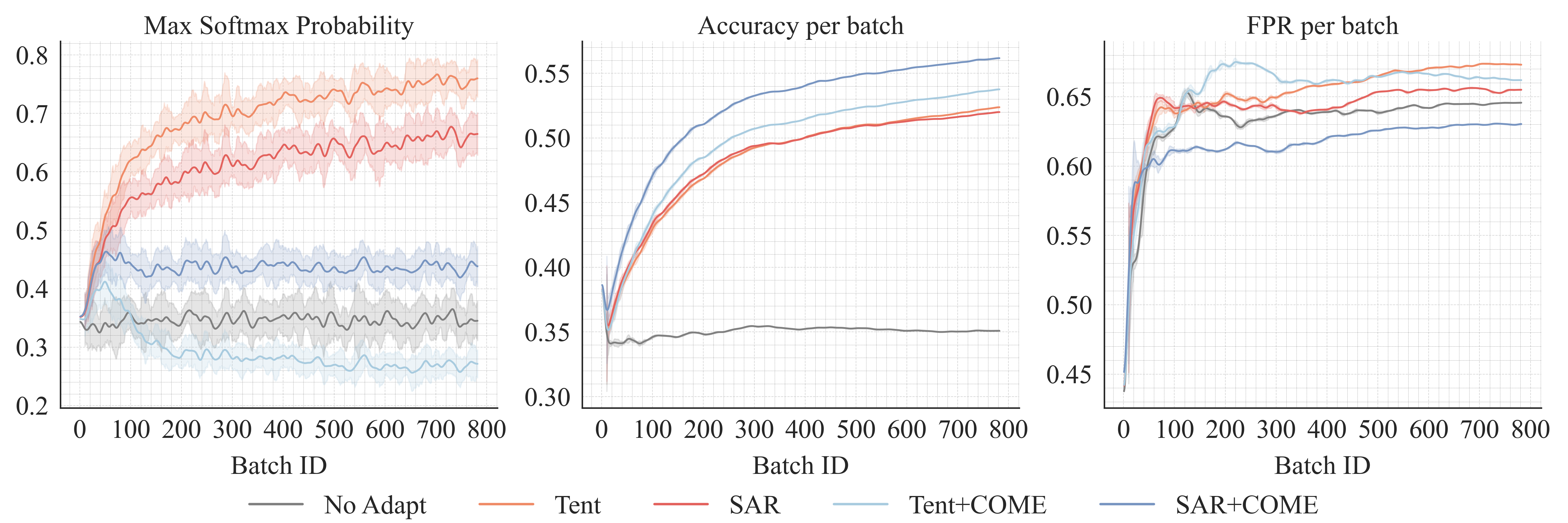}
    \caption{Comparison on two representative TTA methods on ImageNet-C under \textbf{Gaussian Noise} corruption of severity level 5. By contrast to EM, our \texttt{COME} establishes a stable TTA process with consistently improved classification accuracy and false positive rate.}
    \label{fig:gaussian_noise}
    \vspace{-10pt}
\end{figure}

\begin{figure}[htbp]
    \vspace{-2pt}
    \centering
    \includegraphics[width=0.99\textwidth]{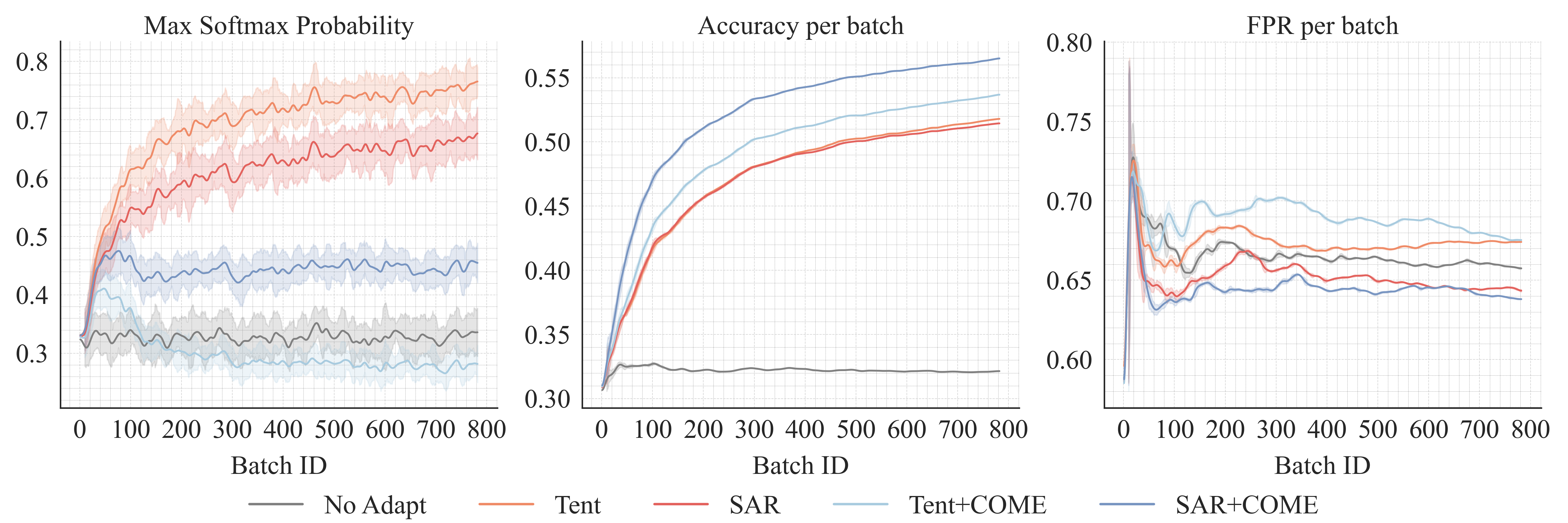}
    \caption{Comparison on two representative TTA methods on ImageNet-C under \textbf{Shot Noise} corruption of severity level 5. By contrast to EM, our \texttt{COME} establishes a stable TTA process with consistently improved classification accuracy and false positive rate.}
    \label{fig:shot_noise}
    \vspace{-10pt}
\end{figure}

\begin{figure}[htbp]
    \vspace{-2pt}
    \centering
    \includegraphics[width=0.99\textwidth]{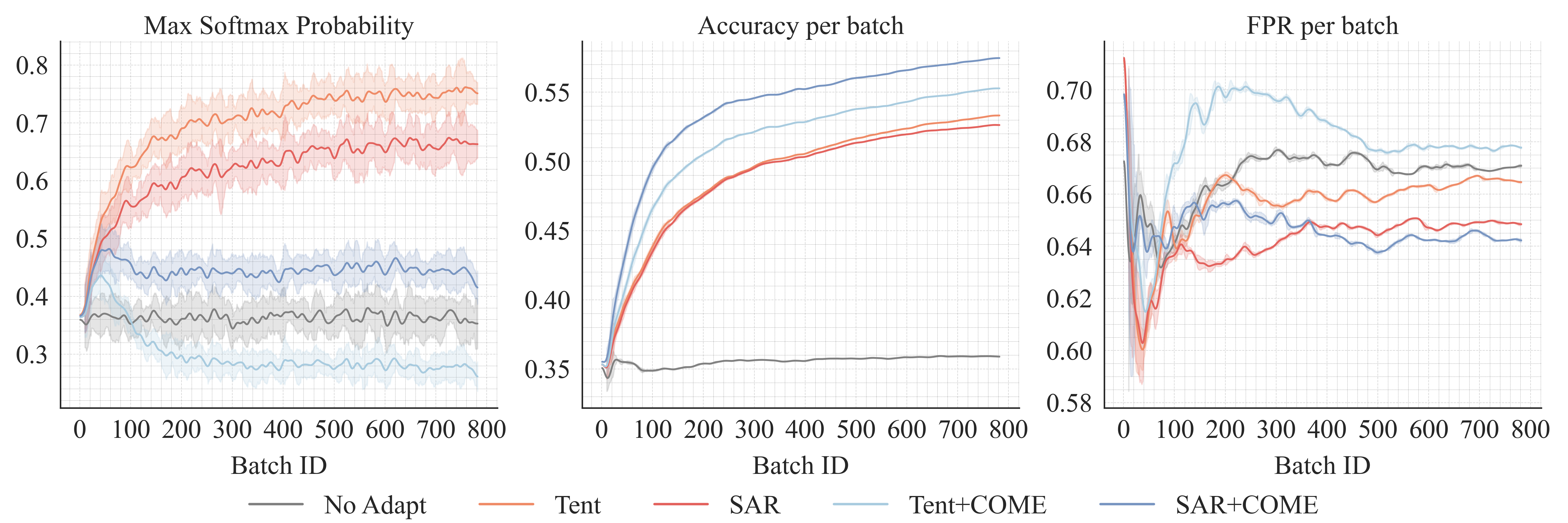}
    \caption{Comparison on two representative TTA methods on ImageNet-C under \textbf{Impulse Noise} corruption of severity level 5. By contrast to EM, our \texttt{COME} establishes a stable TTA process with consistently improved classification accuracy and false positive rate.}
    \label{fig:impulse_noise}
    \vspace{-10pt}
\end{figure}

\begin{figure}[htbp]
    \vspace{-2pt}
    \centering
    \includegraphics[width=0.99\textwidth]{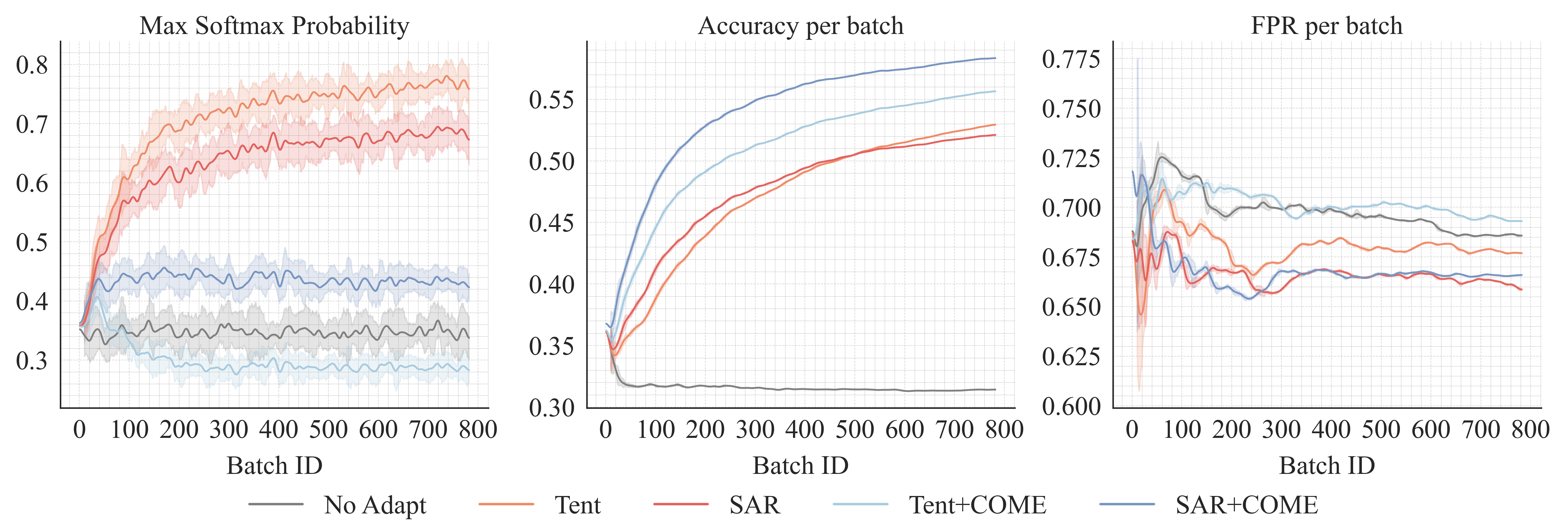}
    \caption{Comparison on two representative TTA methods on ImageNet-C under \textbf{Defocus Blur} corruption of severity level 5. By contrast to EM, our \texttt{COME} establishes a stable TTA process with consistently improved classification accuracy and false positive rate.}
    \label{fig:defocus_blur}
    \vspace{-10pt}
\end{figure}

\begin{figure}[htbp]
    \vspace{-2pt}
    \centering
    \includegraphics[width=0.99\textwidth]{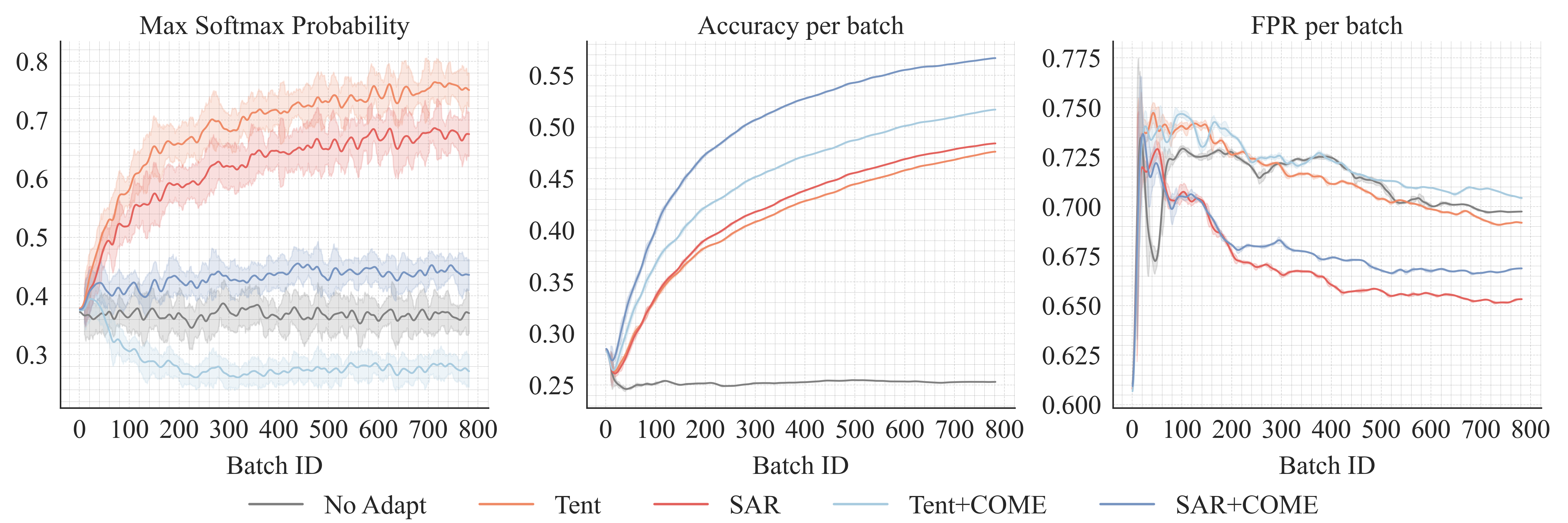}
    \caption{Comparison on two representative TTA methods on ImageNet-C under \textbf{Glass Blur} corruption of severity level 5. By contrast to EM, our \texttt{COME} establishes a stable TTA process with consistently improved classification accuracy and false positive rate.}
    \label{fig:glass_blur}
    \vspace{-10pt}
\end{figure}

\begin{figure}[htbp]
    \vspace{-2pt}
    \centering
    \includegraphics[width=0.99\textwidth]{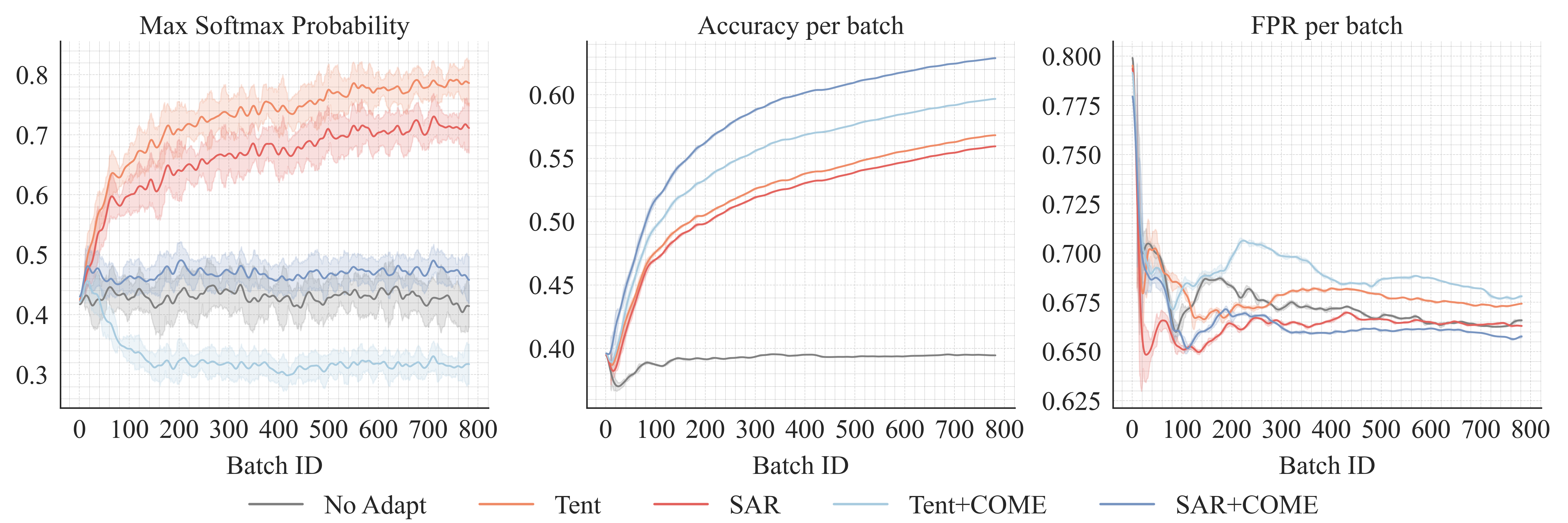}
    \caption{Comparison on two representative TTA methods on ImageNet-C under \textbf{Motion Blur} corruption of severity level 5. By contrast to EM, our \texttt{COME} establishes a stable TTA process with consistently improved classification accuracy and false positive rate.}
    \label{fig:motion_blur}
    \vspace{-10pt}
\end{figure}

\begin{figure}[htbp]
    \vspace{-2pt}
    \centering
    \includegraphics[width=0.99\textwidth]{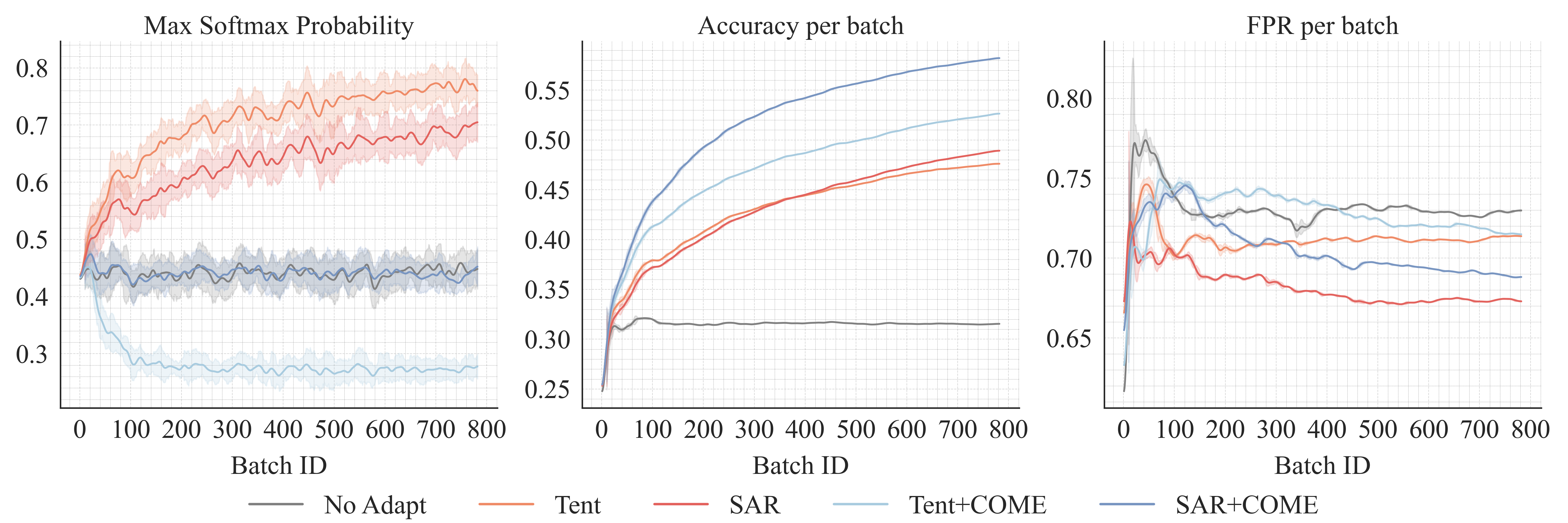}
    \caption{Comparison on two representative TTA methods on ImageNet-C under \textbf{Zoom Blur} corruption of severity level 5. By contrast to EM, our \texttt{COME} establishes a stable TTA process with consistently improved classification accuracy and false positive rate.}
    \label{fig:zoom_blur}
    \vspace{-10pt}
\end{figure}

\begin{figure}[htbp]
    \vspace{-2pt}
    \centering
    \includegraphics[width=0.99\textwidth]{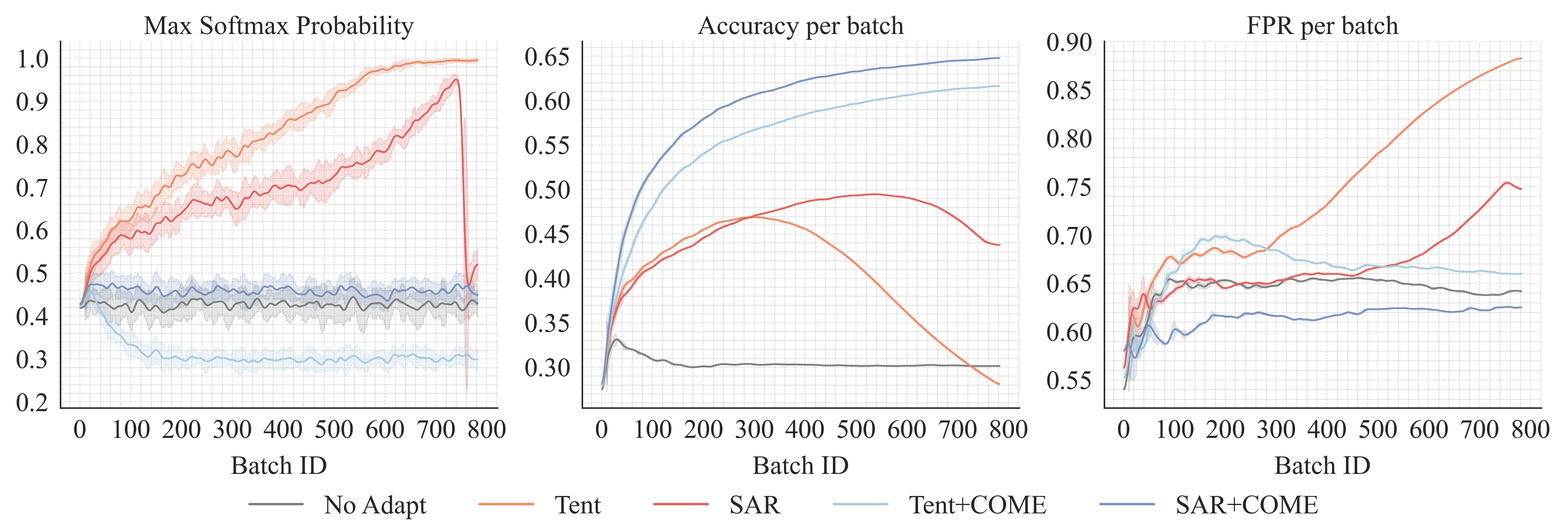}
    \caption{Comparison on two representative TTA methods on ImageNet-C under \textbf{Frost} corruption of severity level 5. By contrast to EM, our \texttt{COME} establishes a stable TTA process with consistently improved classification accuracy and false positive rate. Although the SAR method can recover the model when it collapses to a trivial solution, its performance remains poor. Our \texttt{COME} method addresses the issue of overconfidence that leads to model collapse.}
    \label{fig:frost}
    \vspace{-10pt}
\end{figure}

\begin{figure}[htbp]
    \vspace{-2pt}
    \centering
    \includegraphics[width=0.99\textwidth]{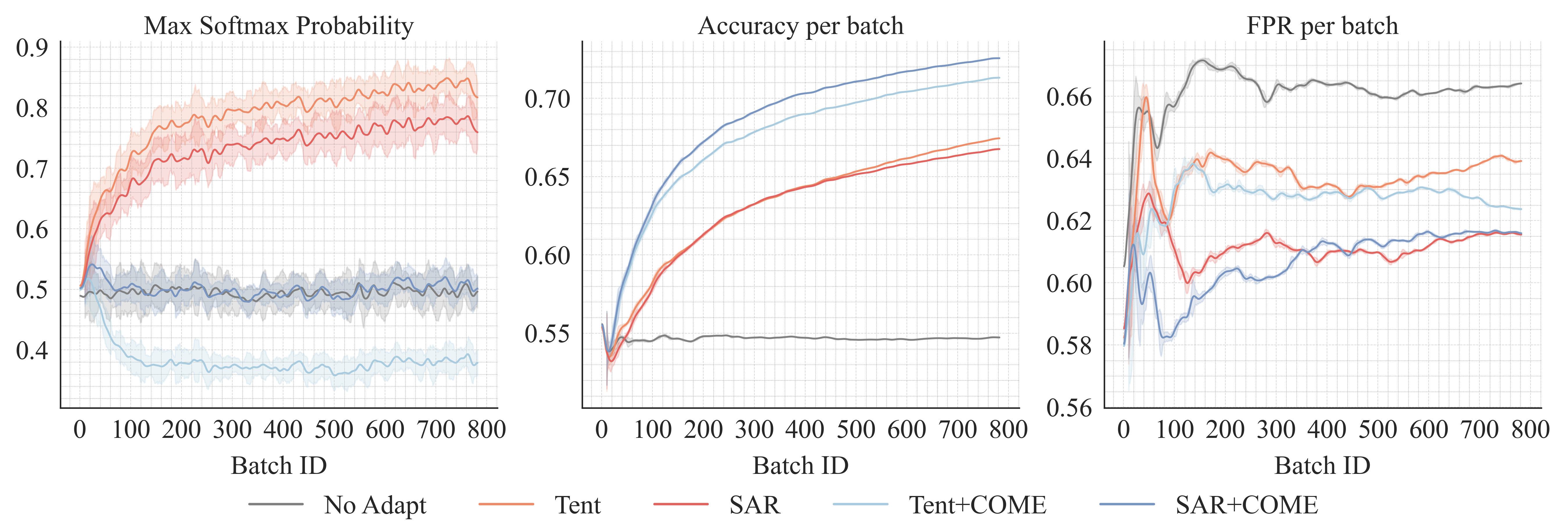}
    \caption{Comparison on two representative TTA methods on ImageNet-C under \textbf{Fog} corruption of severity level 5. By contrast to EM, our \texttt{COME} establishes a stable TTA process with consistently improved classification accuracy and false positive rate.}
    \label{fig:fog}
    \vspace{-10pt}
\end{figure}

\begin{figure}[htbp]
    \vspace{-2pt}
    \centering
    \includegraphics[width=0.99\textwidth]{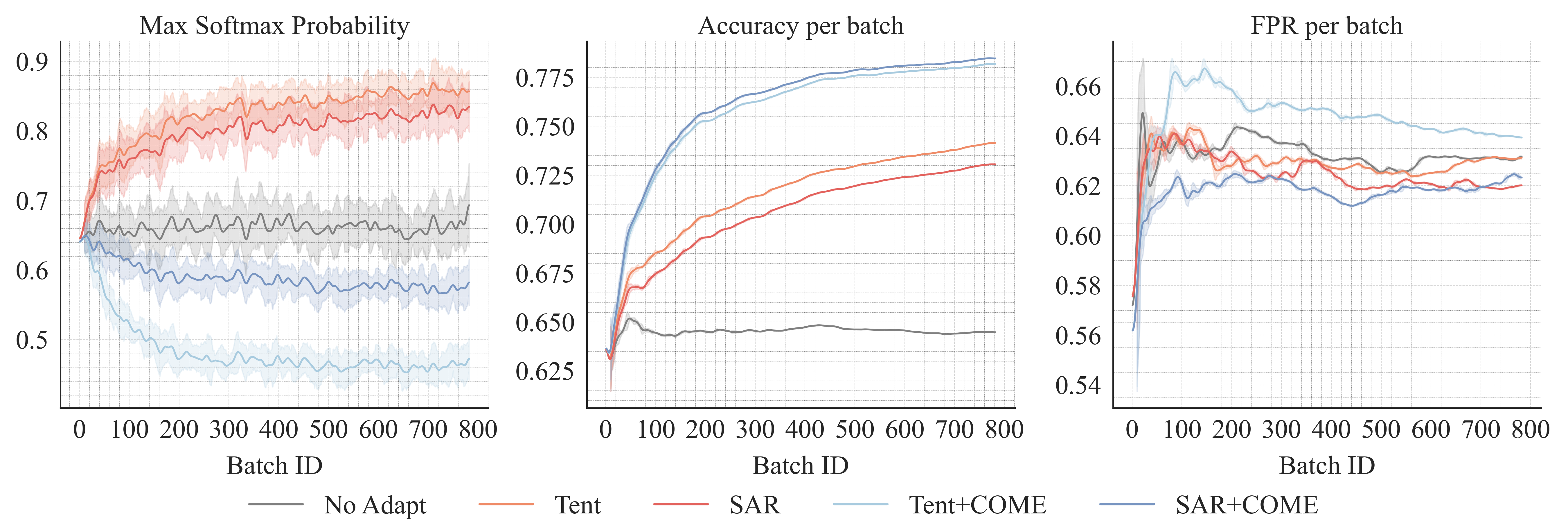}
    \caption{Comparison on two representative TTA methods on ImageNet-C under \textbf{Brightness} corruption of severity level 5. By contrast to EM, our \texttt{COME} establishes a stable TTA process with consistently improved classification accuracy and false positive rate.}
    \label{fig:brightness}
    \vspace{-10pt}
\end{figure}

\begin{figure}[htbp]
    \vspace{-2pt}
    \centering
    \includegraphics[width=0.99\textwidth]{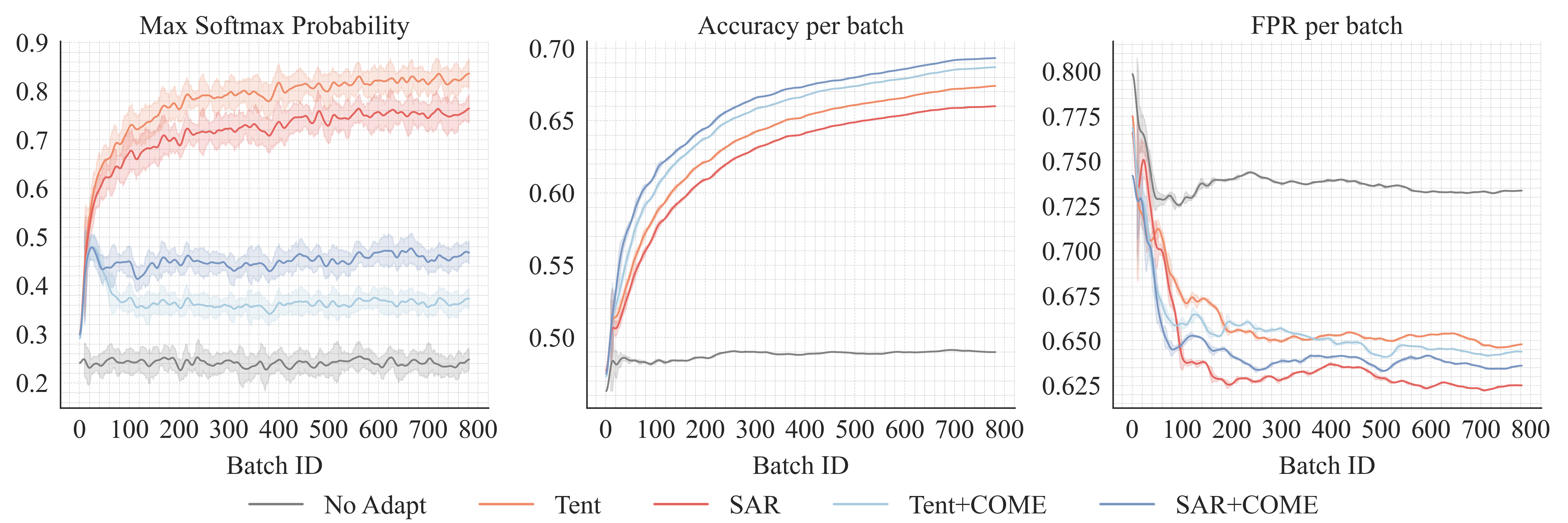}
    \caption{Comparison on two representative TTA methods on ImageNet-C under \textbf{Contrast} corruption of severity level 5. By contrast to EM, our \texttt{COME} establishes a stable TTA process with consistently improved classification accuracy and false positive rate.}
    \label{fig:contrast}
    \vspace{-10pt}
\end{figure}

\begin{figure}[htbp]
    \vspace{-2pt}
    \centering
    \includegraphics[width=0.99\textwidth]{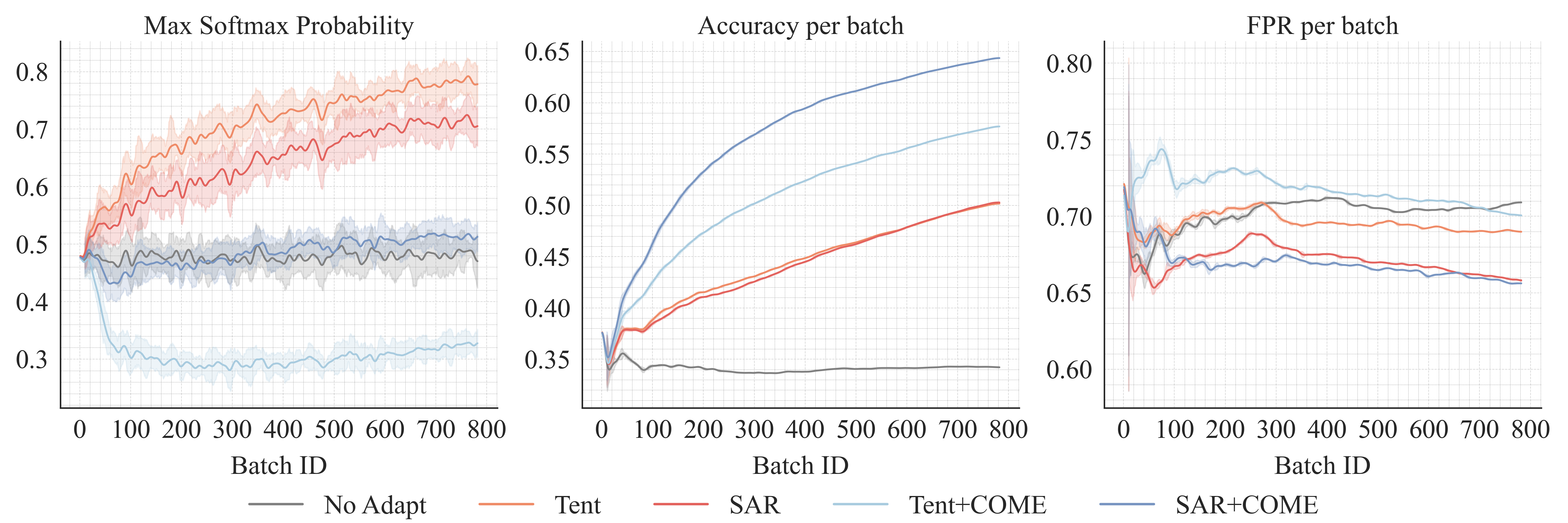}
    \caption{Comparison on two representative TTA methods on ImageNet-C under \textbf{Elastic Transform} corruption of severity level 5. By contrast to EM, our \texttt{COME} establishes a stable TTA process with consistently improved classification accuracy and false positive rate.}
    \label{fig:elastic_transform}
    \vspace{-10pt}
\end{figure}

\begin{figure}[htbp]
    \vspace{-2pt}
    \centering
    \includegraphics[width=0.99\textwidth]{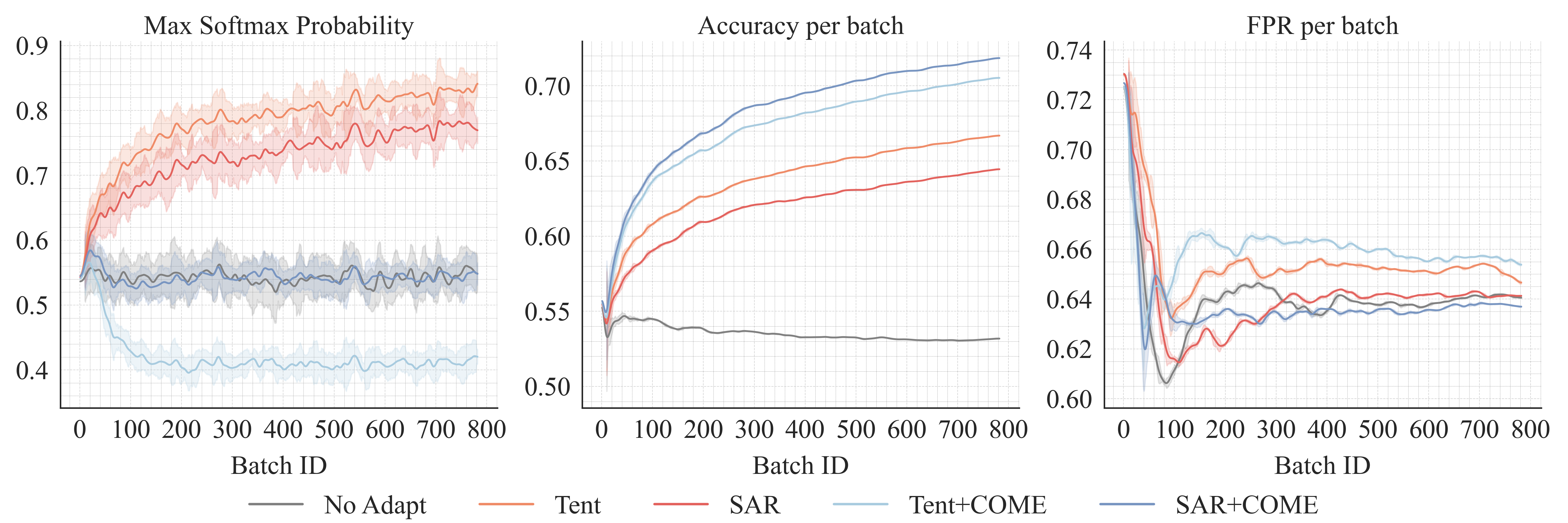}
    \caption{Comparison on two representative TTA methods on ImageNet-C under \textbf{Pixelate} corruption of severity level 5. By contrast to EM, our \texttt{COME} establishes a stable TTA process with consistently improved classification accuracy and false positive rate.}
    \label{fig:pixelate}
    \vspace{-10pt}
\end{figure}

\begin{figure}[htbp]
    \vspace{-2pt}
    \centering
    \includegraphics[width=0.99\textwidth]{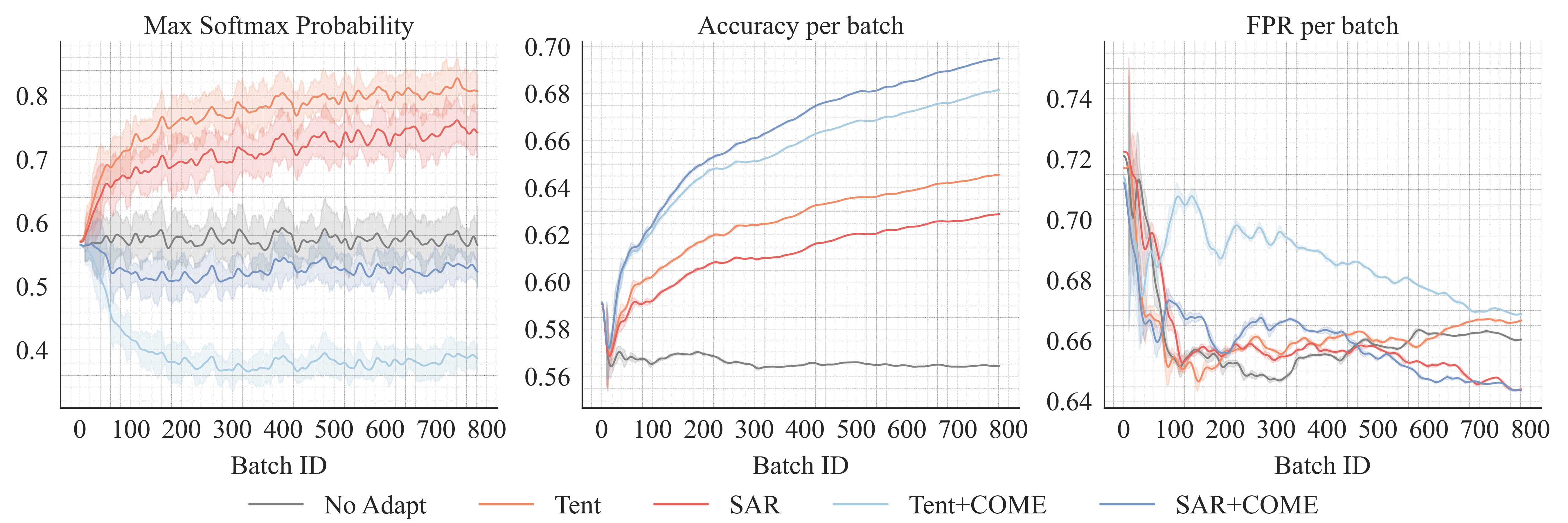}
    \caption{Comparison on two representative TTA methods on ImageNet-C under \textbf{Jpeg Compression} corruption of severity level 5. By contrast to EM, our \texttt{COME} establishes a stable TTA process with consistently improved classification accuracy and false positive rate.}
    \label{fig:jpeg_compression}
    \vspace{-10pt}
\end{figure}

\section{Discussion}

\subsection{Effects of Design Choice}
\textbf{Choices of output function to obtain the parameters of the Dirichlet distribution.} By definitions, the parameters of a Dirichlet distribution $\boldsymbol{\alpha}$ must be greater than 1 and the evidence $\boldsymbol{e}$ should be non-negative. This can be achieved by applying ReLU activation function or exponential function to the output logits as suggested in previous works~\citep{han2022trusted,malinin2018predictive}. That is, we can get the evidence by
\begin{equation}
    \boldsymbol{e}={\rm ReLU}(f(x))
\end{equation}
or
\begin{equation}
    \boldsymbol{e}=\exp f(x)-1.
\end{equation}
In this paper, we choose the exponential function. Since we assume the pretrain model is trained with standard cross-entropy loss, using exponential function to get the evidence can keep the training strategy unchanged. Besides, based on our early empirical findings, using exponential function can achieve better classification performance compared to ReLU.

We refer interested readers to~\citep{malinin2018predictive} and Gal's PhD Thesis~\citep{gal2016uncertainty} for more detailed implementation instructions and math deviations.

\textbf{Choices of uncertainty constraint.} In Lemma 1, we prove that by constraining on the model output logits, we can control the uncertainty mass $u$ not to diverge too far from the pretrained model. Previous work~\citep{wei2022mitigating} proposes to mitigate the overconfidence issue by normalizing the logits during pretrain progress in supervised learning tasks. Following their implementation, we propose to optimize on the direction vector of $f(x)$, i.e., $f(x)/||f(x)||_p$, and thus we can expect that the optimization progress is not related to the magnitude of $f(x)$, i.e., its norm. Different from~\cite{wei2022mitigating}, we recover the magnitude by multiplying the direction vector with its norm (detached), rather than a constant to avoid an additional hyperparameter.

\subsection{Limitations and Future work}
Many state-of-the-art TTA methods are equipped with entropy minimization learning principle;, but the potential pitfalls lie in this optimization objective is not well understood. In this paper, we provide empirical analysis towards understanding the failure mode. These findings motivate us to further explore the connection between uncertainty learning and reliable TTA progress, which further implies a principle to design novel learning principle as an alternative to entropy minimization. Finally, we perform extensive experiments on multiple benchmarks to support our findings. In the work, a simple yet effective regularization on the uncertainty mass is devised, and other regularization techniques could be explored. Another interesting direction is further explore the relationship between overconfidence issue and model collapse theoretically.

\end{document}